\documentclass[9pt]{article}
\usepackage{fullpage}
\usepackage{libertine}

\PassOptionsToPackage{numbers, compress}{natbib}

\usepackage[utf8]{inputenc} %
\usepackage[T1]{fontenc}    %
\usepackage{hyperref}       %
\usepackage{url}            %
\usepackage{booktabs}       %
\usepackage{amsfonts}       %
\usepackage{nicefrac}       %
\usepackage{microtype}      %
\usepackage[dvipsnames]{xcolor} %

\usepackage{placeins}
\usepackage{stfloats}
\usepackage{amsmath}
\usepackage{amssymb}
\usepackage{amsthm}
\usepackage{xspace}
\usepackage{nicefrac}
\usepackage{graphicx}
\usepackage{algorithm}
\usepackage{algorithmicx}
\usepackage[noend]{algpseudocode}
\usepackage[numbers,sort&compress]{natbib}
\usepackage{caption}
\usepackage{subcaption}
\usepackage[capitalize,noabbrev]{cleveref}
\usepackage{wasysym}
\usepackage[symbol]{footmisc}

\newcommand{\define}{\triangleq}

\newcommand{\newterm}{\textit}
\newcommand{\sR}{\mathbb{R}}
\newcommand{\sX}{\mathbb{X}}

\newcommand{\sE}{\mathbb{E}}
\newcommand{\id}{\mathbf{1}}
\newcommand{\E}{\sE}

\newcommand{\featureset}{X}

\newcommand{\feasible}{\mathcal{F}}

\newcommand{\cost}{c}
\newcommand{\rcost}{\overline{c}}
\newcommand{\enc}{\phi}
\newcommand{\modcost}{c}
\newcommand{\graph}{\mathcal{G}}

\newcommand{\posclf}{f}

\newcommand{\decision}{F}
\newcommand{\x}{x}
\newcommand{\initx}{\x}
\newcommand{\advx}{\x^*}

\newcommand{\rx}{\overline{\x}}

\newcommand{\lp}[1]{L_{#1}}

\newcommand{\gain}{g}
\newcommand{\util}{u}

\usepackage{pifont}%
\newcommand{\cmark}{\ding{51}}%
\newcommand{\xmark}{\ding{55}}%

\newtheorem{assumption}{Assumption}[]

\newtheorem{lemma}{Lemma}[]
\newtheorem{statement}{Statement}[]

\def\equationcrefname~#1\null{Equation~(#1)\null}

\DeclareMathOperator*{\argmin}{arg\,min}

\definecolor{divergent1}{rgb}{0.12156862745098039, 0.4666666666666667, 0.7058823529411765}
\definecolor{divergent2}{rgb}{1.0, 0.4980392156862745, 0.054901960784313725}
\definecolor{divergent3}{rgb}{0.17254901960784313, 0.6274509803921569, 0.17254901960784313}
\definecolor{rocket1}{rgb}{0.20973515, 0.09747934, 0.24238489}
\definecolor{rocket2}{rgb}{0.43860848, 0.12177004, 0.34119475}
\definecolor{rocket3}{rgb}{0.67824099, 0.09192342, 0.3504148}

 \newcommand{\kknote}[1]{}
\newcommand{\bknote}[1]{}
\newcommand{\ctnote}[1]{}

\newcommand{\parabf}[1]{\smallskip\noindent\textbf{#1}}
\newcommand{\parait}[1]{\smallskip\noindent\textit{#1}}
\renewcommand{\paragraph}{\parabf}

\title{Adversarial Robustness for Tabular Data through Cost and Utility Awareness
}

\newcommand\CoAuthorMark{\footnotemark[\arabic{footnote}]}
\author{Klim Kireev\footnote{Contributed equally.}, \ Bogdan Kulynych\protect\CoAuthorMark, \  Carmela Troncoso\\EPFL SPRING Lab}
\date{}

\begin{document}

\maketitle

\begin{abstract}
\noindent Many safety-critical applications of machine learning, such as fraud or abuse detection, use data in tabular domains. Adversarial examples can be particularly damaging for these applications. Yet, existing works on adversarial robustness primarily focus on machine-learning models in image and text domains. We argue that, due to the differences between tabular data and images or text, existing threat models are not suitable for tabular domains. These models do not capture that the costs of an attack could be more significant than imperceptibility, or that the adversary could assign different values to the utility obtained from deploying different adversarial examples. We demonstrate that, due to these differences, the attack and defense methods used for images and text cannot be directly applied to tabular settings. We address these issues by proposing new cost and utility-aware threat models that are tailored to the adversarial capabilities and constraints of attackers targeting tabular domains. We introduce a framework that enables us to design attack and defense mechanisms that result in models protected against cost and utility-aware adversaries, for example, adversaries constrained by a certain financial budget. We show that our approach is effective on three datasets corresponding to applications for which adversarial examples can have economic and social implications.
\\[1.4em]
\noindent \textit{This is the extended version of a conference paper appearing in the proceedings of the 2023 Network and Distributed System Security (NDSS) Symposium. Please cite the conference version~\cite{confversion}.}
\end{abstract}

\section{Introduction}\label{sec:intro}
Adversarial examples are inputs deliberately crafted by an adversary to cause a classification mistake. They pose a threat in applications for which such mistakes can have a negative impact on deployed models (e.g., a financial loss~\cite{GhamiziCGPBTG20} or a security breach~\cite{DemontisMBMARCG17, GrosseMP0M17, KolosnjajiDBMGER18}). Adversarial examples also have positive uses. For instance, they offer a means of redress in applications in which classification causes harm to its subjects (e.g., privacy-invasive applications~\cite{JiaG18, KulynychOTG20, AlbertPSS20}).

The literature on adversarial examples largely focuses on image~\cite{SzegedyZSBEGF13, GoodfellowSS14, PapernotMJFCS16, Moosavi-Dezfooli16, CarliniWagner17, MadryMSTV17} and text domains~\cite{YangCHWJ20, WangHBSMLZ20, WangPLL19, LeiWCDDW18, EbrahimiRLD18, LiangLSBLS18}.
Yet, many of the applications where adversarial examples are most damaging or helpful are not images or text. High-stakes fraud and abuse detection systems~\cite{CarminatiSPZ20},
risk-scoring systems~\cite{GhamiziCGPBTG20}, operate on \emph{tabular data}: A cocktail of categorical, ordinal, and numeric features. As opposed to images, each of these features has its own different semantics. For example, in a typical representation of an image, all dimensions of an input vector are similar in their semantics: they represent a color of a pixel. In tabular data, one dimension could correspond to a numeric value of a person's salary, another to their age, and another to a categorical value representing their marital status.

The properties of the image domain have shaped the way adversarial examples and adversarial robustness are approached in the literature~\cite{Moosavi-Dezfooli16} and have greatly influenced adversarial robustness research in the text domain. In this paper, we argue that adversarial examples in tabular domains are of a different nature, and adversarial robustness has a different meaning. Thus, the definitions and techniques used to study these phenomena need to be revisited to reflect the tabular context. 

We argue that two high-level differences need to be addressed. First, imperceptibility, which is the main requirement considered for image and text adversarial examples, is ill-defined and can be irrelevant for tabular data. Second, existing methods assume that all adversarial inputs have the same value for the adversary, whereas in tabular domains different adversarial examples can bring drastically different gains.

\paragraph{Imperceptibility and semantic similarity are not necessarily the primary constraints in tabular domains.}
The existing literature commonly formalizes the concept of ``an example deliberately crafted to cause a misclassification'' as a \newterm{natural example}, i.e., an example coming from the data distribution, that is \emph{imperceptibly} modified by an adversary in a way that the classifier's decision changes. Typically, imperceptibility is formalized as closeness according to a mathematical distance such as $\lp{p}$~\cite{SharifBR18, ZhangSAL19}.

In tabular data, however, imperceptibility is not necessarily relevant. Let us consider the following toy example of financial-fraud detection:
Assume a fraud detector takes as input two features: (1) transaction \textsf{amount}, and (2) \textsf{device} from which the transaction was sent. 
The adversary aims to create a fraudulent financial transaction.
The adversary starts with a natural example (\textsf{amount}=$\$200$, \textsf{device}=`Android phone') and changes the feature values until the detector no longer classifies the example as fraud.
In this example, \emph{imperceptibility is not well-defined}. Is a modification to the \textsf{amount} feature from \$200 to \$201 imperceptible? What increase or decrease would we consider perceptible? The issue is even more apparent with categorical data, for which standard distances such as $\lp{2}$, $\lp{\infty}$ cannot even capture imperceptibility: Is a change of the \textsf{device} feature from Android to an iPhone imperceptible? 
Even if imperceptibility was well-defined, \emph{imperceptibility might not be relevant.} Should we only be concerned about adversaries making ``imperceptible'' changes, e.g., modifying \textsf{amount} from \$200 to \$201? What about attack vectors in which the adversary evades detection while changing the transaction by a ``perceptible'' amount: from \$200 to \$2,000? %

Formalizing adversarial examples as imperceptible modifications narrows the mathematical tools that can be used to study adversarial examples in their broad sense. In the case of tabular data, this prevents the study of techniques that adversaries could employ in ``perceptible'', yet effective ways.

We argue that in tabular data the primary constraint should be \emph{adversarial cost}, rather than any notion of similarity.
Instead of looking at how visually or semantically similar are the feature vectors, the focus should be on \emph{how costly it is for an adversary to enact a modification}. Costs capture the effort of the adversary, e.g., financial or computational. ``How much money does the adversary have to spend to evade the detector?'' better captures the possibility that an adversary deploys an attack than establishing a threshold on the $\lp{p}$ distance the adversary could tolerate. In the fraud-detection example, regardless of whether a change from Android to iPhone is imperceptible and semantically similar or not, it is certain that the change costs the adversary a certain amount of resources. How significant are these costs determines the likelihood of the adversary deploying such an attack.

\paragraph{Different tabular adversarial examples are of different value to the adversary.}
In the literature, with a notable exception of \citet{ZhangE19}, defenses against adversarial examples implicitly assume that all adversarial examples are equal in their importance~\cite{GoodfellowSS14, MadryMSTV17, ZhangYJXGJ19, WongSMK18, ShafahiNG0DSDTG19}. 
In tabular data domains, however, different adversarial examples can bring very different \newterm{gains} to the adversary. In the fraud-detection example, if a fraudulent transaction with transaction \textsf{amount} of \$2,000 successfully evades the detector, it could be significantly more profitable to the adversary than a transaction with \textsf{amount} of \$200. 

Using the adversarial cost as the primary constraint for adversarial examples provides a natural way to incorporate the variability in adversarial gain. 
The adversary is expected to care about the profit obtained from the attack, i.e., the difference between the cost associated with crafting an adversarial example, and the gain from its successful deployment. We call this difference the \newterm{utility} of the attack. We show how utility can be incorporated into the design of attacks to ensure their economic profitability, and into the design of defenses to ensure protection against adversaries that focus on profit.

\smallskip
In this paper, we introduce a framework to study adversarial examples tailored to tabular data. Our contributions: 
\begin{itemize}

\item We propose two \emph{adversarial objectives} for tabular data that address the limitations of the standard approaches: a \emph{cost-bounded} objective that substitutes standard imperceptibility constraints with adversarial costs; and a novel \emph{utility-bounded} objective in which the adversary adjusts their expenditure on different adversarial examples proportionally to the potential gains from deploying them. 

\item We propose a practical attack algorithm based on greedy best-first graph search for crafting adversarial examples that achieve the objectives above.

\item We propose a new method for adversarial training to build classifiers that are practically resistant to adversaries pursuing our adversarial objectives.

\item We empirically evaluate our attacks and defenses in realistic conditions demonstrating their applicability to real-world security scenarios. Our evaluation shows that cost-bounded defenses that ascribe equal importance to every example, as traditional approaches do, can degrade robustness against adversaries for whom some attacks have more value than others. Defenses crafted against these adversaries, however, perform well against both cost-oriented and utility-oriented adversaries.

\end{itemize}

\section{Evasion Attacks}
This section introduces the notation and the formal setup of \newterm{evasion attacks}~\cite[see, e.g., ][]{BiggioR18, PapernotMSW16} in tabular domains.

\paragraph{Feature Space in Tabular Domains.} The input domain's \newterm{feature space} $\sX$ is composed of $m$ features: $\sX \subseteq \sX_1 \times \sX_2 \times \cdots \times \sX_m$. For example $\x \in \sX$, we denote the value of its $i$-th feature as $\x_i$. Features $x_i$ can be categorical, ordinal, or numeric. Each example is associated with a binary class label $y \in \{0, 1\}$.

\paragraph{Target Classifier.} We assume the adversary's \newterm{target} to be a binary classifier $\decision(\x) \in \{0, 1\}$ that aims to predict the class $y$ to which an example $\x$ belongs. It is parameterized by a decision function $\posclf(\x) \in [0, 1]$ such that $\decision(\x) = \id[\posclf(\x) > \nicefrac{1}{2}]$.
The output of $\posclf(\x)$ can be interpreted as a score for $\x$ belonging to the positive class ($y = 1$). 
We focus on binary classification as it is the task in which adversarial dynamics typically arise in tabular domains (e.g., fraud detection~\cite{CarminatiSPZ20} or risk-scoring systems~\cite{GhamiziCGPBTG20}).

\paragraph{Adversarial Examples.}
An evasion attack proceeds as follows: The adversary starts with an initial example $\x \in \sX$ with a label $y = y_s$. We call this class the \emph{adversary's source class}. The adversary's goal is to modify $\x$ to produce an \newterm{adversarial example} $\advx$ that is classified as $\decision(\advx) = y_t$, $y_s \neq y_t$. We call this the \emph{adversary's target class}. The attack is \newterm{successful} if the adversary can produce such an adversarial example. Depending on the adversarial objective, the adversarial example might also need to satisfy additional constraints, as detailed in \cref{sec:statement}.

Because an attack is performed using an adversarial example, as in the literature, we use the terms \textit{adversarial example} and \textit{attack} interchangeably.

Our methods can be used in a multi-class setting as they are agnostic to which class is the target one. Our notation, however, is specific to the binary setting for clarity.

\paragraph{Adversarial Model.} In terms of capabilities, we assume the adversary can only perform modifications that are within the domain constraints. In the fraud-detection example, the adversary can change the transaction amount, but the value must be positive. 
For a given initial labeled example $(\x, y)$, we denote the set of feasible adversarial examples that can be reached within the capabilities of the adversary as $\feasible(\x, y) \subseteq \sX$. 

In terms of knowledge, we assume that the adversary has black-box access to the target classifier: The adversary can issue queries using arbitrary examples and obtain $\posclf(\x)$. In our evaluation (\cref{sec:experiments}) we compare this adversary against existing attacks with white-box access to the gradients.

\paragraph{Preservation of Semantics.} It is common to require that an adversarial example is \newterm{semantics-preserving}~\cite{SharifBR18, Pierazzi2020IntriguingPO}: the adversarial example retains the same true class as the original example.
We do not impose such a requirement. The only constraint we impose is that the modifications leading to an adversarial example \textit{are feasible within the domain constraints}, i.e., that the adversarial example belongs to the set $\feasible(\x, y)$. This is because in tabular domains limiting $\feasible(\x, y)$ to those adversarial examples that also preserve semantics is counterproductive: As long as the adversary successfully achieves their goal with an adversarial example that is feasible and is within their budget (see \cref{sec:statement}), the attack presents a valid threat.

\section{Adversarial Objectives in Tabular Data}
\label{sec:statement}
As we detail in \cref{sec:intro}, the approaches to adversarial modeling tailored to image or text data have two critical limitations when applied to tabular domains:
\begin{enumerate}
    \item[1.] \textit{Focus on imperceptibility and semantic similarity.}
    Neither closeness to natural examples in $L_p$ distance, nor closeness in terms of semantic similarity, is a well-applicable definition of adversarial examples in tabular domains. This is because such similarities are either ill-defined for mixed-type features (e.g., as is the case with $L_p$ distance), or potentially irrelevant to the quantification of adversarial constraints (both $L_p$ and semantic similarity).
    \item[2.] \textit{Assuming all adversarial examples are equally useful.} Most existing defenses against adversarial examples do not distinguish different attacks in terms of their value for the adversary. In tabular domains, due to the inherent heterogeneity of the data, some attacks could bring significantly more gain to the adversary. 
\end{enumerate}
Next, we propose adversarial objectives which aim to address these limitations.

\subsection{Cost-Bounded Objective}
Evasion attacks which use adversarial costs were first formalized in early works on adversarial machine learning~\cite{LowdM05, BarrenoNSJT06}. In these works, the adversary aims to find evading examples with minimal cost. Since the discovery of adversarial examples in computer vision models~\cite{SzegedyZSBEGF13}, this formalization was largely abandoned in favor of constraints based on $\lp{p}$ and other mathematical distances (e.g., Wasserstein distance \cite{wong2019wasserstein} or LPIPS~\cite{laidlaw2021perceptual}\cite{kireev2022effectiveness}).
In this work, we revisit the cost-oriented approach, which better reflects the adversary's capabilities in tabular domains.

In a standard way to obtain an adversarial example~\cite{MadryMSTV17}, the adversary aims to construct an example that maximizes the classification loss $\ell(\cdot, \cdot)$ incurred by the target classifier, while keeping the $\lp{p}$-distance from the initial example bounded:
\begin{equation}\label{eq:adv-example-classical}
    \max_{\x' \in \feasible(\x, y)} \ell(\posclf(\x'), y) \quad \text{ s.t. } \|\x' - \x\|_p \leq \varepsilon 
\end{equation}

This objective implicitly assumes that the adversary wants to keep the adversarial example as similar to the initial example as possible in terms of the examples' feature values. The closeness in terms of $\lp{p}$ distance aims to capture imperceptibility and to preserve the original example's semantics~\cite{SharifBR18}.

\paragraph{From Distances to Costs.} To address the fact that imperceptibility or semantic similarity is not necessarily relevant for adversarial settings in tabular domains, we adapt the definition in \cref{eq:adv-example-classical} to the tabular setting by introducing a cost constraint. 

This constraint represents the limited amount of resources available to the adversary to evade the target classifier. If the adversary can find an adversarial example that achieves this goal within the cost budget, the adversary proceeds with the attack.
Formally, we associate a cost to the modifications needed to generate any adversarial example $\x' \in \feasible(\x, y)$ from the original example $(x, y)$. We encode this cost as a function $\cost: \sX \times \sX \rightarrow \sR^+$. We assume the generation cost is zero if and only if no change is enacted: $\cost(\x, \x') = 0 \iff x = x'$.

This formulation is generic: it can encompass geometric and semantic distances, but it goes beyond that. It exhibits the following desirable properties:
\begin{enumerate}
\item[a.]\textit{Support for arbitrary feature types and rich semantics.} Whereas $L_p$ distances only support numeric features, our generic cost model can support any feature type. This is because it does not enforce any structural constraints on the exact form of cost of changing a feature value $\x_i$ into $\x_i'$. For example, the cost does not need to obey $|x_i - x_i'|$ as would be the case with $L_1$ distance. Moreover, unlike mathematical distances, our model does not require the costs to be symmetric. For instance, an increase in a feature value could have a different cost than a decrease.

\item[b.] \textit{Enables more generic quantification of adversarial effort.} Our cost model imposes neither a geometric structure such as is the case with $L_p$ distances, nor any ties to semantic similarity. Thus, the costs can be quantified in those units that are directly relevant to adversarial constraints. An important use case is that our model supports defining costs in the financial sense, i.e., assigning a dollar cost to mounting an attack with a given adversarial example as opposed to semantic closeness or closeness in feature space.

\item[c.] \textit{Support for feature-level accumulation.} Related literature on attacks in tabular data often formalizes costs using indepenent per-feature constraints (see \cref{sec:related}). Although our generic cost model supports such a special case, it also enables accumulation of per-feature costs. Therefore, it can encode a realistic assumption that changing more features increases adversary's expenditure.

\end{enumerate} 
 
\paragraph{The Optimization Problem.} We assume that the cost-bounded adversary has a budget $\varepsilon$. The adversary aims to find any example that flips the classifier's decision \emph{and} that is within the cost budget:
\begin{equation}\label{eq:cost-constrained-hard}
    \max_{\x' \in \feasible(\initx)} \id[ \decision(\x') \neq y ] \quad\text{s.t. $\cost(\x, \x') \leq \varepsilon$}
\end{equation}
Alternatively, the adversary can optimize a standard surrogate objective which ensures that the optimization problem can be solved in practice:
\begin{equation}\label{eq:cost-constrained-soft}
    \max_{\x \in \feasible(\initx, y)} \ell(\posclf(\x), y) \quad\text{s.t. $\cost(\initx, \x') \leq \varepsilon$},
\end{equation}

In the surrogate form, the optimization problem of the cost-bounded adversary is an adaptation of \cref{eq:adv-example-classical} with the norm constraint substituted by the adversarial-cost constraint.
This formalization is in line with recent formalizations of adversarial examples~\cite{MadryMSTV17}, as opposed to early approaches which aim to find minimal-cost attacks~\cite{LowdM05}. This enables us to reuse tools from the recent literature on adversarial robustness to design defenses (see \cref{sec:defenses}).

\subsection{Utility-Bounded Objective}
\label{sec:ub-obj}
The cost-bounded adversarial objective solves the issue of imperceptibility and semantic similarity not being suitable constraints for tabular data. It does not, however, tackle the problem of heterogeneity of examples: the adversary cannot assign different importance to different adversarial examples. In a realistic environment, it can be a serious drawback. For instance, an adversary might spend more resources than they gain from a successful attack. Another instance is the defender hypothetically suffering serious losses due to high-impact adversarial examples, even if for the majority of examples the defense is appropriate.

We propose to capture this heterogeneity by introducing the \newterm{gain} of an attack. The gain, $\gain: \sX \rightarrow \sR^+$, represents the reward (e.g., the revenue) that the adversary receives if their attack using a given adversarial example is successful.

We also introduce the concept of \newterm{utility}: the net benefit of deploying a successful attack. We define the utility $\util_{\initx, y}(\advx)$ of an attack mounted with adversarial example $\advx$ as simply the gain minus the costs:
\begin{equation}\label{eq:utility-defn}
    \util_{\initx, y}(\advx) \define \gain(\advx) - \cost(\initx, \advx),
\end{equation}
where $(\x, y)$ is the initial example. 

Recall that the adversary has black-box access to the target classifier. Thus, they can learn whether an example $\advx$ evades the classifier or not (i.e., whether $\decision(\advx) \neq y$). Then, they can decide to deploy an attack with an adversarial example $\advx$ only if the utility of the attack exceeds a given \newterm{margin} $\tau \geq 0$. Otherwise, the adversary discards this adversarial example. 
Formally, we can model this process by using a \emph{utility constraint} instead of a cost constraint:
\begin{equation}\label{eq:util-constrained-var-gain-hard}
    \max_{\x \in \feasible(\initx, y)} \id[\decision(\x) \neq y] \quad\text{s.t. $\util_{\initx, y}(\x) \geq \tau$}
\end{equation}

If we assume that the gain is constant for any adversarial example $\x' \in \feasible(\x, y)$ that is a modification of an initial example $(\initx, y)$, $\gain(\initx) = \gain(\x')$, this problem can also be seen as a variant of the cost-bounded formulation in \cref{eq:cost-constrained-hard}, where $\varepsilon$ varies for different initial examples:
\begin{equation}\label{eq:util-constrained-hard}
    \begin{aligned}
    \max_{\x' \in \feasible(\x, y)} & \id[\decision(\x') \neq y] \\
    \text{s.t. } & \util_{\x, y}(\x') \geq \tau \\
    \iff & \gain(x) - \cost(\x, \x') \geq \tau \\
    \iff & \cost(\x, \x') \leq \varepsilon(x) \define \gain(x) - \tau
    \end{aligned}
\end{equation}

In \cref{app:extra-discussion}, we discuss the formalization of a utility-maximization objective which models an adversary which wants to maximize their profit subject to budget constraints.

\subsection{Quantifying Cost and Utility}
A natural question in our setup is how to define the adversary's costs and gains. This question is relevant to all related prior work on adversarial robustness in tabular data (see \cref{sec:related}). For example, if adversarial robustness is defined in terms of an $L_p$ distance, both the attacker and defender need to determine an acceptable perturbation magnitude, which inherently comes from domain knowledge.

In our applications (see \cref{sec:experiments}), we focus on the settings in which adversarial capabilities are constrained in terms of financial costs. In such settings, we expect that the adversary is able to quantify the financial costs $c(\x, \x')$ and gains $g(x)$ by practical necessity. On the defender's side, estimating these values is trickier, as the defender might be unaware of the exact capabilities of the adversary. The defender thus needs to employ standard threat modeling techniques and domain knowledge. It is worth mentioning that the defender is not required to estimate the capabilities perfectly. Rather, they need to obtain the lower bound on the adversary's costs. After that, if the defended system is robust, it is robust against the adversary whose costs are at least as high as estimated.

In our utilitarian approach, it is possible to include other concerns and constraints of the adversary as part of the utility definition by measuring them in the same units as the utility (e.g., financial costs). For instance, as the driving concern behind imperceptibility-based approaches to adversarial robustness is the detection of an attack, the gain could be adjusted for a potential risk of being detected. The adversary could estimate the probability of being detected (e.g. using public statistics), and incorporate it into the gain by subtracting an expected value of the attack failure due to detection.

\section{Finding Adversarial Examples in Tabular Domains}
\label{sec:attacks}
In this section, we propose practical algorithms for finding adversarial examples suitable to achieve the adversarial objectives we introduce in~\cref{sec:statement}.

\subsection{Graphical Framework}
\label{sec:attack-algos}
The optimization problems in \cref{sec:statement} can seem daunting due to the large cardinality of $\feasible(\x, y)$ when the feature space is large. To make the problems tractable, we transform them into graph-search problems, following the approach by \citet{KulynychHST18}. Consider a \newterm{state-space graph} $\graph(\initx) = (V, E)$. Each node corresponds to a feasible example in the feature space, $V = \feasible(\x, y) \cup \{\initx\}$. Edges between two nodes $\x$ and $\x'$ exist if and only if they differ in value of one feature: there exists $i \in [n]$ such that $\x_i \neq \x'_i$, and $\x_j = \x'_j$ for all $j \neq i$. In other words, the immediate descendants of a node in the graph consist of all feasible feature vectors that differ from the parent in exactly one feature value.

Using this state-space graph abstraction, the objectives in \cref{sec:statement} can be modeled as graph-search problems. Even though the graph size is exponential in the number of feature values, the search can be efficient. This is because it can construct the relevant parts of the graph on the fly as opposed to constructing the full graph in advance.

Building the state-space graph is straightforward when features take discrete values. To encode continuous features in the graph we discretize them by only considering changes to a continuous feature $i$ that lie  within a finite subset of its domain $\sX_i$, in particular, on a discrete grid. The search efficiency depends on the size of the grid. As the grid gets coarser, finding adversarial examples becomes easier. This efficiency comes at the cost of potentially missing adversarial examples that are not represented on the grid but could fulfil the adversarial constraints with less cost or higher utility.

\subsection{Attacks as Graph Search}
\label{sec:assumptions}
In the remainder of the paper, we make the following assumptions about the adversarial model:

\begin{assumption}[Modular costs] The adversary's costs are \newterm{modular}: they decompose by features. Formally, changing the value of each feature $i$ from $\x_i$ to $\x'_i$ has the associated cost $\modcost_i(\x_i, \x'_i) > 0$, and the total cost of modifying $\x$ into $\x'$ is a sum of individual feature-modification costs:
\begin{equation}\label{eq:mod-costs}
    \cost(\x, \x') = \sum_i^n \modcost_i(\x_i, \x'_i)
\end{equation}
\end{assumption}
The state-space graph can encode modular costs by assigning weights to the graph edges. An edge between $\x$ and $\x'$ has an associated weight of $\modcost_i(\x_i, \x_i')$, where $i$ is the index of the feature that differs between $\x$ and $\x'$. For pairs of examples $\x^{(0)}$ and $x^{(t)}$ that differ in more than one feature, the cost $\cost(\x^{(0)}, x^{(t)})$ is the sum of the edge costs along the shortest path from $\x^{(0)}$ to $x^{(t)}$.

\begin{assumption}[Constant gain] For any initial example $(x, y)$, the adversary cannot change the gain:
\begin{equation}
    \forall x' \in \feasible(x, y): \quad \gain(x) = \gain(x') 
\end{equation}
\end{assumption}

This follows the approach in utility-oriented strategic classification (as detailed in \cref{sec:ub-obj}). This assumption is not formally required for our attack algorithms (described next in this section), but we focus on this setting in our empirical evaluations. In \cref{sec:var-gain}, we experimentally show that removing this assumption does not significantly affect our results.

\begin{algorithm}[t]
    \caption{Best-First Search (BFS)}
    \label{algo:bfs}
    \begin{algorithmic}[1]
    \Function{$\mathsf{BFS}_{B, s, \varepsilon}$}{$\initx$}
      \State $\mathtt{open} \gets \textsc{MinPriorityQueue}_B(\initx, 0)$
      \State $\mathtt{closed} \gets \{\}$
      \While{$\mathtt{open}$ is not empty}
        \State $v \gets \mathtt{open}.\textsc{pop}()$
        \If{$v \notin \mathsf{closed}$}
            \State $\textsc{closed} \gets \textsc{closed} \cup \{v\}$
        \EndIf
        \If{$\eta(v) \geq \delta$} \Return $v$ \EndIf
        \State $S \gets \textsc{expand}(v)$
        \For{$t \in S$}
            \If{$t \notin \mathtt{closed}$ and $\cost(\initx, t) \leq \varepsilon$} \State $\mathtt{open}.\textsc{add}(t, s(v, t))$
            \EndIf
        \EndFor
      \EndWhile
    \EndFunction
  \end{algorithmic}
\end{algorithm}

\paragraph{Strategies to Find Adversarial Examples.}
Under our two assumptions, the cost-bounded objective in \cref{eq:cost-constrained-hard} and the utility-bounded objective in \cref{eq:util-constrained-hard} can be achieved by finding any adversarial example that is classified as target class and is within a given cost bound. Thus, these adversarial goals can be achieved using \newterm{bounded-cost search}~\cite{SternPF11}. 

We start with the \newterm{best-first search} (BFS)~\cite{HartNR68, KulynychHST18}, a flexible meta-algorithm that generalizes many common graph search algorithms. In its generic version (\cref{algo:bfs}) BFS keeps a bounded priority queue of \emph{open nodes}. It iteratively pops the node $v$ with the highest score value from the queue (best first), and adds its immediate descendants to the queue. This is repeated until the queue is empty. The algorithm returns the node with the highest score out of all popped nodes.

The BFS algorithm is parameterized by the \newterm{scoring function} $s: V \times V = \sX \times \sX$ and the size of the priority queue $B$. Different choices of the scoring function yield search algorithms suited for solving different graph-search problems, such as Potential Search for bounded-cost search~\cite{SternPF11, SternFBPSG14}, and A$^*$~\cite{Korf85, DechterP85} for finding the minimal-cost paths. When $B = \infty$, the algorithm might traverse the full graph and is capable of returning the optimal solution. As the size of $B$ decreases, the optimality guarantees are lost. When $B = 1$ BFS becomes a \newterm{greedy} algorithm that myopically optimizes the scoring function. When $1 < B < \infty$ we get a \newterm{beam search} algorithm that keeps $B$ best candidates at each iteration.

To achieve the adversarial objectives in \cref{sec:statement}, we propose to use a concrete instantiation of BFS, what we call the \textit{Universal Greedy (UG)} algorithm. Inspired by heuristics for cost-bounded optimization of submodular functions~\cite{KhullerMN99, Wolsey82}, we set the scoring function to balance the increase in the classifier's score and the cost of change:
\begin{equation}\label{eq:ug-scoring-func}
    s(v, t) = -\frac{\posclf(t) - \posclf(v)}{\cost(v, t)} \,
\end{equation}
The minus sign appears because BFS expands the lowest scores first, and we need to maximize the score. We set the beam size to $B = 1$ (greedy), which enables us to find high-quality solutions to \textit{both} cost-bounded and utility-bounded problems at reasonable computational costs (see \cref{sec:experiments}).

\section{Defending from Adversarial Examples in Tabular Domains}
\label{sec:defenses}

A common way to mitigate the risks of adversarial examples is adversarial training~\cite{GoodfellowSS14, MadryMSTV17}. 
In a standard approach~\cite{MadryMSTV17}, these adversarial examples are constructed during training by modifying natural examples $x$ with perturbations constrained in an $\lp{p}$-ball of radius $\varepsilon$ as in \cref{eq:adv-example-classical}.
As explained in \cref{sec:intro,sec:statement}, this approach, however, does not apply to the tabular domains. 

Another difference between the image and tabular domains is the efficiency of generating adversarial examples. In images, adversarial examples used for training are generated using efficient methods such as Projected Gradient Descent (PGD)~\cite{MadryMSTV17} or the Fast Gradient Sign Method (FGSM)~\cite{GoodfellowSS14, WongRK20}. These algorithms produce adversarial examples fast, and enable the efficient implementation of adversarial training. Fast generation, however, is not possible for tabular domains. The algorithms to produce tabular adversarial examples introduced in \cref{sec:attacks} require thousands of inference operations using the target model. Although generating one example can take seconds, generating thousands of adversarial examples required during adversarial training quickly becomes infeasible for computationally constrained defenders. 
To make the generation of adversarial examples feasible during adversarial training, we introduce approximate versions of the attacks that rely on a relaxation of initial attack constraints.

\subsection{Relaxing the Constraints}\label{sec:defenses-relax}
Following the setting of the standard Projected Gradient Descent (PGD) method~\cite{MadryMSTV17}, adversarial training for the cost-bounded adversary could be defined as the following optimization objective:
\begin{equation}\label{eq:cost-constrained-hard-defense}
   \min_{\theta} \max_{\x' \in \feasible(\initx, y)} \ell(\posclf_\theta(\x'), y) \quad\text{s.t. $\cost(\initx, \x') \leq \varepsilon$},
\end{equation}
where $\posclf_{\theta}$ is a parametric classifier and $\theta$ are its parameters.

To keep the computational requirements low, we relax the problem to optimize over a convex set, which enables us to adapt the PGD method.
Let us define $B_\varepsilon$ to be the constraint region of \cref{eq:cost-constrained-hard-defense}:
\[
    B_\varepsilon(\x, y) \define \{ \x' \in \feasible(x, y) \mid \cost(\x, \x') \leq \varepsilon\}
\]
We construct a relaxation of $B_\varepsilon$ in two steps:
\[
    B_\varepsilon \xrightarrow[(1)]{} \bar B_\varepsilon  \xrightarrow[(2)]{}  \tilde B_\varepsilon
\]

\begin{enumerate}
    \item[(1)] \textit{Continuous relaxation.} We map $B_\varepsilon$ into a continuous space using an \newterm{encoding function} $\phi: \sX \rightarrow \sR^d$, and a \newterm{relaxed cost function} $\bar \cost: \sR^d \times \sR^d \rightarrow \sR^+$.
    The continuous relaxation is defined as:
    \begin{equation}
        \bar B_\varepsilon \define \{ \phi(x') \mid \bar \cost(\phi(x), \phi(x')) \leq \varepsilon\},
    \end{equation}
    where $x' \in \feasible(x, y)$.
    The pair $(\phi, \bar \cost)$ is designed to satisfy the following condition:
    \begin{equation}\label{eq:relaxed-cost}
        \forall x' \in B_\varepsilon(x, y): \, \bar \cost(\phi(x), \phi(x')) \leq \cost(x, x'),
    \end{equation}
    ensuring that every example $x' \in B_\varepsilon(x, y)$ is mapped to an element in the relaxed set, $\enc(x') \in \bar B_\varepsilon$.
    We denote the encoded value $\enc({x})$ as $\rx$.
    
    \item[(2)] \textit{Convex cover.} To enable adversarial training using PGD, we need that elements of the relaxed set can be projected onto the constraint region. For this purpose, we cover $\bar B_\varepsilon$ with a convex superset $\tilde B_\varepsilon$, e.g., a convex hull of $\bar B_\varepsilon$.
    The convex superset $\tilde B_\varepsilon$ needs to be constructed such that there exists an efficient algorithm for \newterm{projection}: For a given $(\x, y)$, and a point $t \in \sR^d$, we want to be able to efficiently solve
    $
        \min_{t' \in \tilde B_\varepsilon(\x, y)} \| t - t' \|_2.
    $
\end{enumerate}

\paragraph{Encoding and Cost Functions.}
As we assume that the cost of modifications is modular (see \cref{sec:assumptions}), we define the encoding and cost functions to be modular too:
$$
\enc(\x) = [{\enc_1}(\x_1),...,{\enc_n}(\x_n)]
$$
$$
\rcost(\enc({\x}), \enc({\x}')) = \sum_{i=1}^n \rcost_i(\enc_i({\x_i}), \enc_i({\x_i}'))
$$

With this formulation, the problem of constructing suitable $\enc$ and $\rcost$ functions is reduced to finding $\enc_i: \sX_i \rightarrow \sR^{d_i}$ and $\rcost_i$ for each feature, where $d_i$ is the dimensionality of the $i$-th feature after encoding. If for all $i \in [m]$ both $\enc_i$ and $\rcost_i$ individually fulfill the requirement in \cref{eq:relaxed-cost}:
\[ \rcost_i(\enc_i(x_i), \enc_i(x'_i)) \leq \cost_i(x_i, x_i'),\]
then the modular cost $\rcost(\rx, \rx')$ fulfills \cref{eq:relaxed-cost} as well. In a slight abuse of notation, we consider that $\rx = \phi(\x)$ is a concatenated $d$-dimensional vector, where $d = \sum_{i = 1}^m d_i$.

In the following we introduce concrete $\enc$ and  $\rcost$ functions for categorical and numeric features.

\parait{Categorical features.}
As the encoding function $\enc_i$ for categorical features we use standard one-hot encoding.
As the relaxed cost function for categorical features we define $\rcost_i$ as follows:
\[\rcost_i({\rx}_i, {\rx}'_i) = \min_{t \in \feasible_i(x, y)} \cost_i({\x}_i, t) \cdot \frac{1}{2} \|\bar \x_i - \bar \x'_i\|_1,\]
where $\feasible_i(x,y)$ is the set of feasible values of the feature $i$. 
For example, suppose that $\x_i$ is a categorical feature with 4 possible values $\featureset_i = \{a,b,c,d\}$, and the minimal cost of change is 2. When $\x_i=b$ and $\x'_i=c$ (corresponding to $\rx_i = (0, 1, 0, 0)$ and  $\rx'_i = (0, 0, 1, 0)$ after one-hot encoding), then $\rcost_i({\rx}_i, {\rx}'_i) = \frac{1}{2}\cdot2\cdot2 = 2$. Note that in this case we use $\rx_i = \enc_i(\x_i) \in \sR^{d_i}$ to mean a $d_i$-dimensional one-hot vector for simplicity of notation.

This cost function enables us to perform the two-step relaxation described before. First, it satisfies \cref{eq:relaxed-cost}, and therefore the constraint region $\bar B_\varepsilon$ includes all mapped examples of $B_\varepsilon$. Second, we can obtain the convex superset $\tilde B_\varepsilon$ as a continuous $\lp{1}$ ball around the mapped values $\bar \x \in \bar B_\varepsilon$. 

\parait{Numeric features.}
A numeric feature is a feature with values belonging to an ordered subset of $\sR$ (e.g. integer, real). In most cases, the identity function (${\enc_i}(\x_i) = x_i$) is sufficient for numerical features. However, more complex encoding functions could also be desirable. For example, when one needs to reduce numerical errors, which can be achieved by normalizing the feature values to $[-1, 1]$, or when the cost is non-linear.

In general, projecting onto arbitrary sets can be challenging. Specifically, the constraint region $B_\varepsilon$ could be non-convex.
We therefore must limit the scope of possible adversarial cost functions that we can model during adversarial training to those that are compatible with efficient projection.
For this, we introduce a cost model that covers a broad class of functions for which $c_i(\x_i, \x'_i)$ can be expressed as $K_i \cdot |\psi(x_i) - \psi(x'_i)|$, where $K_i$ is a constant and $\psi(x)$ is an invertible function.

For instance, this model covers the following exponential cost model: $c(x, x') = K \cdot |e^{x} - e^{x'}|$. In this case, we can encode the features as $\enc_i(x_i) = \psi^{-1}(x_i) \define \ln(x_i)$. This transformation enables us to account for certain non-linear cost functions $\cost$ with respect to the input space using linear cost functions $\rcost$ in the relaxed space $\bar B_\varepsilon$.

We define the relaxed cost function for numerical features as a piecewise-linear function, with different coefficients for increasing or decreasing the feature value:
\begin{equation}
\rcost_j(\rx_j, \rx'_j) = c_{j-}(\x) \cdot [\bar \x_j - \bar \x_j']^{+} + c_{j+}(\x) \cdot [\bar \x_j' 
- \bar \x_j]^{+}
\end{equation}
where $[t]^{+}$ returns $t$ if $t > 0$, and $0$ otherwise, and $c_{j-}(\x)$ and $c_{j+}(\x)$ encode the costs for decreasing and increasing the value of the feature $j$, respectively, and can vary from one initial example $\x$ to another.

Note that in this model the final cost of a modification could depend on the way in which this modification is achieved. A direct modification from $x$ to $x''$ could have different cost than first modifying $x$ to $x'$ and then $x'$ to $x''$, i.e., $\rcost(\bar \x, \bar \x'') \neq \rcost(\bar \x, \bar \x') + \rcost(\bar \x', \bar \x'')$.

\parait{Total cost.} Given the set of categorical feature indices, $\mathcal{C}$, and the set of numeric feature indices, $\mathcal{I}$, the total relaxed cost function is:
\begin{equation}\label{eq:final-cost-function}
\rcost(\rx, \rx') = \sum_{i \in \mathcal{C}} \min_{t \in \feasible_i(x, y)} \cost_i({\x}_i, t) \cdot \frac{1}{2} \|\bar \x_i - \bar \x'_i\|_1 
+ \sum_{j \in \mathcal{I}} c_{j-}(\x)  \cdot [\bar \x_j - \bar \x_j']^{+} + c_{j+}(\x) \cdot [\bar \x_j' - \bar \x_j]^{+}
\end{equation}

\subsection{Adversarial Training with Projected Gradient Descent}
Using the cost model introduced before, we redefine the training optimization problem in \cref{eq:cost-constrained-hard-defense} to generate adversarial examples over a specific instantiation of the convex set $\tilde B_\varepsilon$, as follows:
\begin{equation}\label{eq:cost-constrained-hard-defense-relax}
   \min_{\theta} \max_{\tilde x' \in \tilde B_\varepsilon(\x, y)} \ell(\posclf_\theta(\tilde x'), y),
\end{equation}
where we specify $\tilde B_\varepsilon$ as:
\begin{equation}\label{eq:convex-superset}
    \tilde B_\varepsilon(\x, y) \define \{ \bar \x + \delta \mid \delta \in \sR^d \land \rcost(\bar \x, \bar \x + \delta) \leq \varepsilon  \}.
\end{equation}
Thus, we can rewrite \cref{eq:cost-constrained-hard-defense-relax}:
\begin{equation}\label{eq:cost-constrained-hard-defense-relax-2}
    \begin{aligned}
    \min_{\theta}\quad & \max_{\delta \in \sR^d} \ell(\posclf_\theta(\rx + \delta), y) \\
    & \text{ s.t. } \rcost(\bar \x, \bar \x + \delta) \leq \varepsilon
    \end{aligned}
\end{equation}
This objective can be optimized using standard PGD-based adversarial training~\cite{MadryMSTV17}. Due to the construction of our cost function in \cref{eq:final-cost-function}, we can use existing algorithms for projecting onto a weighted $\lp{1}$-ball~\cite{slav2010ieee,perez2020efficient} with an appropriate choice of weights. As these approaches are standard, we omit them in the main body, and provide the details in~\cref{sec:proj-details}.
 
\subsection{Adversarial Training against a Utility-Bounded Adversary}
For the utility-bounded adversary we propose to use an objective similar to \cref{eq:cost-constrained-hard-defense-relax-2}, but applying individual constraints to different examples:
\begin{equation}\label{eq:cost-constrained-hard-defense-relax-ub}
\begin{aligned}
   \min_{\theta} \quad & \max_{\delta \in \sR^d} \ell(\posclf_\theta(\rx + \delta), y) \\
   & \text{ s.t. } \rcost(\rx, \rx + \delta) \leq \varepsilon(\x) \define [\bar \gain(\rx) - \tau]_+,
\end{aligned}
\end{equation}
where $\bar \gain(\rx)$ is such that $\bar \gain(\rx) = \gain(x)$.
In this formulation, we apply our assumption of invariant gain (see \cref{sec:assumptions}) as $\bar \gain(\rx + \delta) = \bar \gain(\rx)$.

This objective aims to decrease the adversary's utility by focusing the protection on examples with high gain.
The main difference with respect to the cost-constrained objective in \cref{eq:cost-constrained-hard-defense-relax-2} is that here we use a different cost bound for different examples $\varepsilon(\x)$. This formulation enables us to directly use the PGD-based adversarial training to defend against utility-bounded adversaries as well.

\section{Experimental Evaluation}
\label{sec:experiments}
In this section, we show that our graph-based attacks can be used by adversaries to obtain profit, and that our proposed defenses are effective at mitigating damage from these attacks.

\subsection{Experimental Setup}

\newcommand{\ieeecis}{$\mathtt{IEEECIS}$\xspace}
\newcommand{\twitterbot}{$\mathtt{TwitterBot}$\xspace}
\newcommand{\homecredit}{$\mathtt{HomeCredit}$\xspace}

\subsubsection{Datasets} 
We perform our experiments on three tabular datasets which represent real-world applications for which adversarial examples can have social or economic implications:
\begin{itemize}
    \item \twitterbot~\cite{GilaniKC17}. The dataset contains information about more than 3,400 Twitter accounts either belonging to humans or bots. The task is to detect bot accounts. We assume that the adversary is able to purchase bot accounts and interactions through darknet markets, thus modifying the features that correspond to the account age, number of likes, and retweets.
    \item \ieeecis~\cite{ieeecis_kaggle}. The dataset contains information about around 600K financial transactions. The task is to predict whether a transaction is fraudulent or benign. We model an adversary that can modify three features for which we can outline the hypothetical method of possible modification, and estimate its cost: payment-card type, email domain, and payment-device type.
    \item \homecredit~\cite{homecredit_kaggle}. The dataset contains financial information about 300K home-loan applicants. The main task is to predict whether an applicant will repay the loan or default. We use 33 features, selected based on the best solutions to the original Kaggle competition~\cite{homecredit_kaggle}. Of these, we assume that 28 can be modified by the adversary, e.g., the loan appointment time. %
\end{itemize}

\subsubsection{Models} 
We evaluate our attacks against three types of ML models commonly applied to tabular data. First, an $L_2$-regularized \textit{logistic regression (LR)} with a regularization parameter chosen using 5-fold cross-validation. Second, \textit{gradient-boosted decision trees (XGBT)}. Third, \textit{TabNet}~\cite{arik2021tabnet}, an attentive transformer neural network specifically designed for tabular data. We optimize the number of steps as well as the capacity of TabNet's fully connected layers using grid search.

\subsubsection{Adversarial Features}
We assume that the feasible set consists of all positive values of  numerical features and all possible values of categorical features. For simplicity, we avoid features with mutual dependencies and treat the adversarially modifiable features as independent. We detail the choice of the modifiable features and their costs in \cref{sec:cost-models}.

\subsubsection{Metrics}
To evaluate the effectiveness of the attacks and defenses, we use three main metrics:
\begin{itemize}
    \item \textit{Adversary's success rate:} The proportion of correctly classified examples from a test set $X_\mathsf{test}$ for which adversarial examples successfully generated using the attack algorithm $\mathcal{A}(x, y)$ evade the classifier:
    \[
        \Pr_{(\x, y) \sim X_\mathsf{test}} [ \decision(\mathcal{A}(x, y)) \neq y \land \decision(\x) = y]\,.
    \]
    \item \textit{Adversarial cost:} Average cost of successful adversarial examples:
    \[
        \E_{(\x, y) \sim X_\mathsf{test}} [ c(x, \mathcal{A}(x, y)) \mid \decision(\mathcal{A}(x, y)) \neq y \land \decision(\x) = y]\,.
    \]
    \item \textit{Adversarial utility:} Average utility (see \cref{eq:utility-defn}) of successful adversarial examples:
    \[
        \E_{(\x, y) \sim X_\mathsf{test}} [u_{x, y}(\mathcal{A}(x, y)) \mid \decision(\mathcal{A}(x, y)) \neq y \land \decision(\x) = y].
    \]
\end{itemize}
In all cases, we only consider correctly classified initial examples which enables us to distinguish these security metrics from the target model's accuracy.
We introduce additional metrics in the experiments when needed.

\subsection{Attacks Evaluation}\label{sec:attack-eval}
We evaluate the attack strategy proposed in \cref{sec:attacks} in terms of its effectiveness, and empirically justify its design.

\subsubsection{Design Choices of the Universal Greedy Algorithm}
When designing attack algorithms in the BFS framework (see \cref{algo:bfs}), there are two main design choices: the scoring function, and the beam size. We explore different configurations and show that our parameter choices for the Universal Greedy attack produce high-quality adversarial examples.

\parait{Beam size. } We define the beam size of the Universal Greedy attack to be one. The other options that we evaluate are 10 and 100.
We evaluate them by running three types of attacks: cost-bounded for three cost bounds $\varepsilon$, and utility-bounded at the breakeven margin $\tau = 0$. %
We compute two metrics: attack success, and the success-to-runtime ratio. This ratio represents how much time is needed to achieve the same level of success rate using each choice of the beam size. This metric is more informative for our evaluation than runtime, as the runtime is simply proportional to the beam size.

For feasibility reasons, we use two datasets: \twitterbot and \ieeecis. We aggregate the metrics across the three models (LR, XGBT, TabNet), and report the average. The results on \twitterbot are equivalent to the results on \ieeecis, thus for conciseness we only report \ieeecis results.

We find that the success rates are equal up to the percentage point for all choices of the beam size. We show the detailed numeric results in \cref{tab:beam-size} in the Appendix. As the smallest beam size of one is the fastest to run, it demonstrates the best success/time ratio, therefore, is the best choice.

\parait{Scoring function.} Recall from \cref{eq:ug-scoring-func} that the scoring function is the cost-weighted increase in the target classifier's confidence,
which aims to maximize the increase in classifier confidence at the lowest cost.

Other choices for the scoring function $s(v,t)$ could be:
\begin{itemize}
    \item \textit{$A^*$ algorithm}~\cite{Korf85, DechterP85, KulynychHST18}:
    $s(v, t) = \cost(v, t) + \lambda \cdot h(t),$
    where $h(t)$ is a heuristic function, which estimates the remaining cost to a solution and $\lambda > 0$ is a greediness parameter~\cite{Pohl70}. This scoring function balances the current known cost of a candidate and the estimated remaining cost.
    We choose the model's confidence for the positive class, $h(\x) = \posclf(x)$, as a heuristic function. Intuitively, this works as a heuristic, because the lower the confidence for the positive class, the more likely we are close to a solution: an example classified as the target class.
    \item \textit{Potential Search} (PS)~\cite{SternPF11, SternFBPSG14}:
    $s(v, t) = \nicefrac{h(t)}{(\varepsilon - c(v, t))}\,,$
    which additionally takes into account the cost bound $\varepsilon$, thus becoming more greedy (i.e., optimizing $s(v, t) = \lambda \cdot h(t)$ with $\lambda \approx \nicefrac{1}{\varepsilon}$) when the cost of the current candidate leaves a lot of room within the $\varepsilon$ budget. We also choose $h(\x) = \posclf(x)$ as a heuristic function.
    \item \textit{Basic Greedy}:
    $s(v, t) = \nicefrac{-\posclf(t)}{c(s, t)}\,,$
    which aims to maximize the classifier's confidence, yet balance it with the incurred cost. Unlike \cref{eq:ug-scoring-func}, this scoring function does not take into account the relative increase of the confidence, only its absolute value.
\end{itemize}

We evaluate the choice of the scoring function on the \twitterbot and \ieeecis datasets, with the beam size fixed to one. We run the cost-bounded and utility-bounded attacks in the same configuration as before, and measure two metrics averaged over the models: Attack success, and attack success/time ratio.

\cref{tab:scoring-func} shows the results. On \ieeecis, the Universal Greedy outperforms the other choices in terms of success rate and the success/time ratio. On the \twitterbot dataset, it outperforms the other choices in the utility-bounded and unbounded attacks. For cost-bounded attacks, the Universal Greedy offers very close performance to the best option, the Basic Greedy.

\begin{table*}[]
    \caption{\textit{Effect of the scoring-function choice} for graph-based attacks. In the majority of settings, our Universal Greedy scoring function offers the best success rate and performance.}
    \label{tab:scoring-func}
    \centering
    \begin{subtable}[t]{.45\linewidth}
        \caption{\ieeecis}
        \centering
        \begin{tabular}{lrrrr}
        \toprule
        {} & \multicolumn{4}{l}{Adv. success, \%} \\
        Cost bound $\rightarrow$    &               10 &  30 &  Gain &   $\infty$ \\
        Scoring func. $\downarrow$  &                    &        &       &        \\
        \midrule
        UG              &        \textbf{45.32} & \textbf{56.57} & \textbf{56.22} & \textbf{68.20} \\
        A*              &        42.37 & 55.62 & 55.34 & 53.47 \\
        PS              &        \textbf{45.32} & 55.14 & 56.18 &   N/A \\
        Basic Greedy    &        42.37 & 55.46 & 55.38 & 53.82 \\
        \bottomrule
        \end{tabular}
        \\[1em]
        \begin{tabular}{lrrrr}
        \toprule
        {} & \multicolumn{4}{l}{Success/time ratio} \\
        Cost bound $\rightarrow$    &               10 &  30 &  Gain &   $\infty$ \\
        Scoring func. $\downarrow$  &                    &        &       &        \\
        \midrule
        UG              &               \textbf{3.78} & \textbf{4.80} & \textbf{2.53} & \textbf{2.06} \\
        A*              &               3.29 & 3.83 & 1.89 & 1.15 \\
        PS              &               \textbf{3.78} & 4.01 & 2.26 &  N/A \\
        Basic Greedy    &               3.21 & 3.86 & 2.01 & 1.16 \\
        \bottomrule
        \end{tabular}
    \end{subtable}
    \quad
    \begin{subtable}[t]{.45\linewidth}
        \caption{\twitterbot}
        \centering
        \begin{tabular}{lrrrr}
        \toprule
        {} & \multicolumn{4}{l}{Adv. success, \%} \\
        Cost bound $\rightarrow$    &               1,000 &  10,000 &  Gain &   $\infty$ \\
        Scoring func. $\downarrow$  &                    &        &       &        \\
        \midrule
        UG              &        80.24 & \textbf{85.35} & \textbf{21.63} & \textbf{87.00} \\
        A*              &        77.56 & 84.45 & 20.29 & 86.25 \\
        PS              &        79.95 & 85.19 & 21.48 &   N/A \\
        Basic Greedy    &        \textbf{80.40} & 85.04 & \textbf{21.63} & 86.85 \\
        \bottomrule
        \end{tabular}
        \\[1em]
        \begin{tabular}{lrrrr}
        \toprule
        {} & \multicolumn{4}{l}{Success/time ratio} \\
        Cost bound $\rightarrow$    &               1,000 &  10,000 &  Gain &   $\infty$ \\
        Scoring func. $\downarrow$  &                    &        &       &        \\
        \midrule
        UG              &             208.95 & 205.76 & \textbf{64.99} & \textbf{205.31} \\
        A*              &             206.33 & 201.93 & 62.25 & 201.31 \\
        PS              &             205.85 & 203.18 & 63.76 &    N/A \\
        Basic Greedy    &             \textbf{210.20} & \textbf{206.20} & 64.32 & 204.96 \\
        \bottomrule
        \end{tabular}
    \end{subtable}
\end{table*}

\subsubsection{Graph-Based Attacks vs. Baselines}\label{sec:graph_vs_pgd}
We compare the Universal Greedy (UG) algorithm against two baselines:  previous work, and the minimal-cost adversarial examples.

\parait{Previous Work: PGD.} As our cost model differs from the existing approaches to attacks on tabular data, we fundamentally cannot perform a fully apples-to-apples comparison against existing attacks (see \cref{sec:attacks}). To compare against the high-level ideas from prior work, we follow the spirit of the attack by \citet{ballet2019imperceptible}, which modifies the standard optimization problem from \cref{eq:adv-example-classical} to use correlation-based weights. We adapt the standard $L_1$-based PGD attack~\cite{MadryMSTV17, maini2020adversarial} to (1) support categorical features through discretization, and (2) use weighted $L_1$ norm following our derivations in \cref{sec:defenses-relax}. We provide a detailed description of this adaptation in \cref{algo:pgd-alg} in \cref{sec:experiments-extra}.

We run attacks using PGD with 100 and 1,000 steps, and compare it to UG (\cref{sec:attacks}) on the \twitterbot and \ieeecis datasets. As PGD can only operate on differentiable models, in this comparison we only evaluate the performance of the attacks against TabNet.

We run the cost-bounded attacks using two values of the $\varepsilon$ bound, specific to each dataset (see \cref{sec:experiments-extra} for the exact attack parameters). As before, we also run a utility-bounded attack at the breakeven margin $\tau = 0$.
We measure the success rates of the attacks, as well as the average cost of the obtained adversarial examples.
For conciseness, we do not report the results on \twitterbot, as they find they are equivalent to those on \ieeecis. 

\begin{figure}[t]
    \centering
    \includegraphics[width=0.6\linewidth]{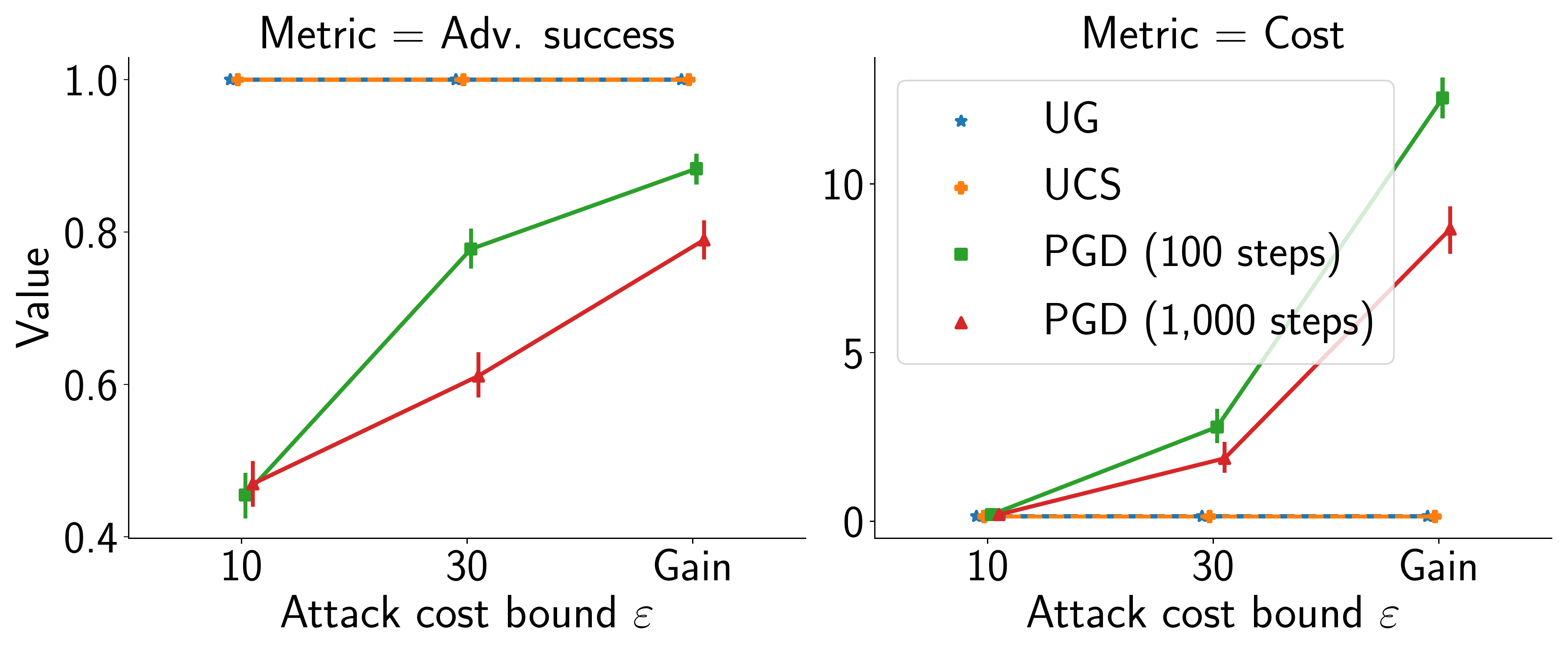}
    \caption{\textit{Universal Greedy attack vs Baselines.} Left: Attack success rate (higher is better for the adversary). Right: Attack cost (lower is better for the adversary). For all cost bounds, our graph-based attack outperforms standard PGD and returns close to optimal-cost adversarial examples (obtained with Uniform-Cost Search, UCS).}
    \label{fig:pgd-comparison}
\end{figure}

\cref{fig:pgd-comparison} shows that the UG attack consistently outperforms the PGD-based baseline both in terms of the success rate and the costs. Our attacks are superior even when the PGD-based baseline produces feasible adversarial examples.

\parait{Minimal-Cost Adversarial Examples.} As UG is a greedy algorithm, we additionally evaluate how far are the obtained adversarial examples from the optimal ones in terms of cost. For this, we compare the results from UG to a standard Uniform-Cost Search (UCS)~\cite{KulynychHST18}. UCS is an instantiation of the BFS framework (see \cref{sec:attacks}) with unbounded beam size, and the scoring function equal to the cost: $s(v, t) = c(v, t)$. In our setting, UCS is guaranteed to return optimal solutions to the following optimization problem:
\begin{equation}\label{eq:min-cost-problem}
    \min_{\x' \in \feasible(\x, y)} c(\x, \x') \quad\text{ s.t. }\decision(\x') \neq y
\end{equation}
\cref{fig:pgd-comparison} shows that UG has almost no overhead over the minimal-cost adversarial examples on TabNet ($1.03\times$ overhead on average). In fact, the average and median cost overhead is $1.80\times$ and $1\times$ over all models, respectively. There exist some outlier examples, however, with over $100\times$ cost overhead. We provide more information on the distribution of cost overhead in \cref{sec:experiments-extra}.

\subsubsection{Performance against Undefended Models}
Having shown that the attacks outperform the baseline, and the design choices are sound, we demonstrate that the attacks bring some \emph{utility} to the adversary. In this section, we evaluate the attacks in a non-strategic setting: the models are not deliberately defended against the attacks. For conciseness, we only evaluate cost-bounded attacks, as the next section provides an extensive demonstration of utility-bounded attacks.

In \emph{all} evaluated settings, the attacks have non-zero success rates and achieve non-zero adversarial utility. \cref{fig:attacks-regular} show the results of cost-bounded attacks for \ieeecis and \homecredit datasets. %
We omit the results for LR on \homecredit as this model does not perform better than the random baseline. An average adversarial example obtained using the cost-bounded objective brings a profit of \$125 to the adversary when attacking the \ieeecis TabNet model, and close to $100\%$ of examples in the test data can be turned into successful adversarial examples. 

Although for all models we see non-zero success and utility, some models are less vulnerable than others, even without any protection. For example, the success rate of the adversary against LR on \ieeecis is much lower than against TabNet (at least 50 p.p. lower). This model, however, is also comparatively inaccurate, with only $62\%$ classification accuracy.

\begin{figure*}
    \captionsetup[subfigure]{justification=centering}
    \centering
    \begin{subfigure}[t]{.62\linewidth}
    \centering
    \includegraphics[width=\linewidth]{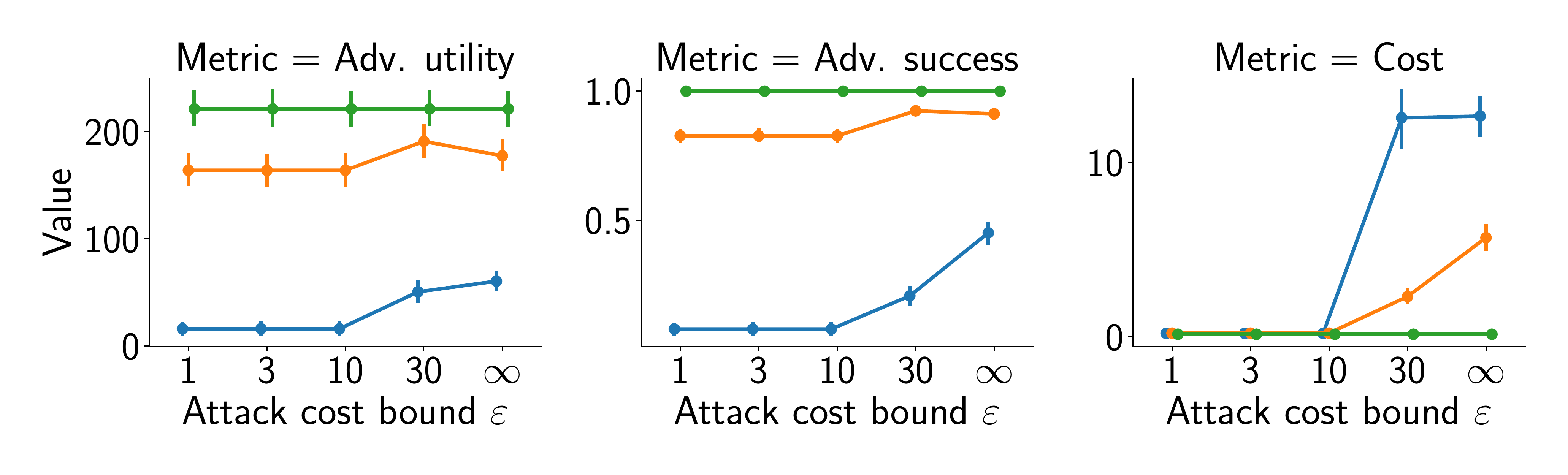} \\
    \vspace{-.5em}
    \caption{\ieeecis. Model (test acc.): 
    \textcolor{divergent1}{$\bullet$}~LR (0.62) \, \textcolor{divergent2}{$\bullet$}~XGBT (0.83) \, \textcolor{divergent3}{$\bullet$}~TabNet (0.77)}
    \vspace{1em}
    \end{subfigure}
    \begin{subfigure}[t]{.62\linewidth}
    \hspace{-1.1em}
    \includegraphics[width=\linewidth]{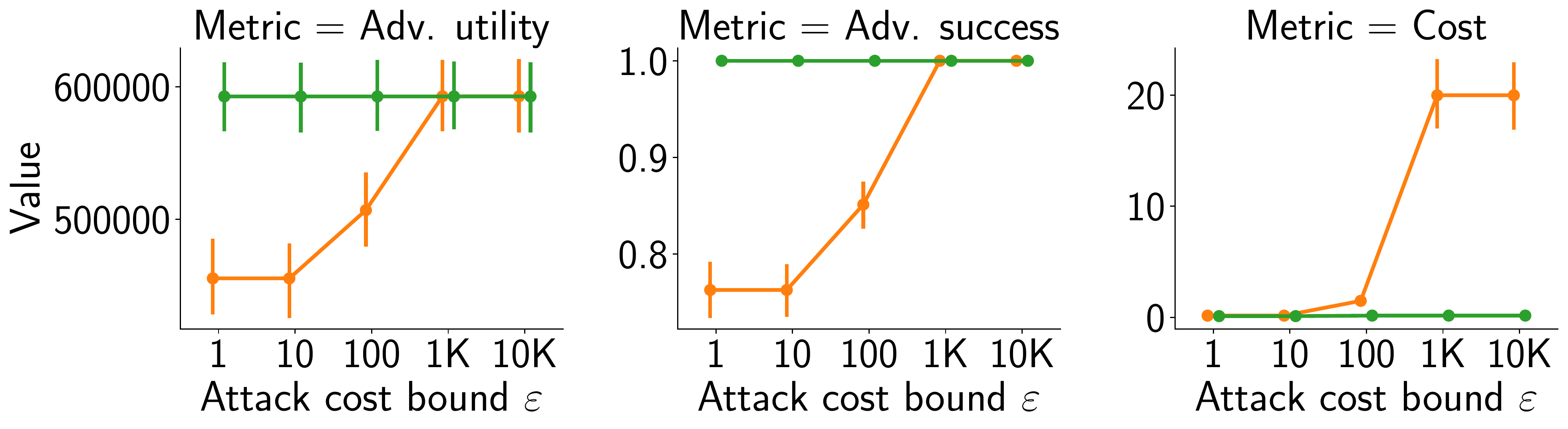}
    \vspace{-.1em}
    \caption{\homecredit. Model (test acc.): 
    \textcolor{divergent2}{$\bullet$}~XGBT (0.65)\, \textcolor{divergent3}{$\bullet$}~TabNet (0.68)}
    \end{subfigure}
    \caption{Results of cost-bounded graph-based attacks against three types of models. Left pane: Adversarial utility (higher is better for the adversary). Middle and right panes: See \cref{fig:pgd-comparison}. 
    On \ieeecis, the attack can achieve utility from approximately \$10 to \$125 per attack depending on the target model. On \homecredit, the average utility ranges between \$400,000 and \$600,000.}
    \label{fig:attacks-regular}
\end{figure*}

\subsection{Evaluation of Our Defense Methods}\label{sec:defense-eval}
We evaluate the defense mechanisms proposed in \cref{sec:defenses} in two scenarios. First, a scenario in which the adversary's objective used by the defender for adversarial training---cost-bounded (CB) or utility-bounded (UB)---matches the attack that will be deployed by the adversary.
Second, a scenario in which the defender models the adversary's objective incorrectly, and uses a different attack than the adversary when performing adversarial training.

\parabf{Baselines.} We set two comparison baselines which provide boundaries for which a defense can be considered effective.
On the accuracy side, we consider the \textit{clean baseline}: a model trained without any defense. It provides the best accuracy, but also the least robustness. Any defense that does not achieve at least the clean baseline's \emph{robustness} should not be considered, as the clean baseline would always provide better or equal accuracy, thus a better robustness-accuracy trade-off.
On the robustness side, we consider the \textit{robust baseline}: a model for which all features that can be changed by the adversary are masked with zeroes for training and testing. As this removes any adversarial input, this model is invulnerable to attacks within the assumed adversarial models. Any practical defense must outperform the robust baseline in terms of \emph{accuracy}. Otherwise, the robust baseline would provide a better trade-off. 

\cref{tab:baselines}
shows the clean and robust baselines' accuracy for the three datasets. On \twitterbot, the robust baseline performs almost as well as the clean model. As there is no space for a better defense for \twitterbot, we only evaluate our defenses for the \ieeecis and \homecredit models.

We train our attacks and defenses using the parameters listed in \cref{tab:hpar-homecredit} in \cref{sec:experiments-extra}.

\begin{table}[t]
     \centering
     \small
    \caption{Baseline performance in terms of average test accuracy.}
    \begin{tabular}{@{}lccc}
    \textbf{} & \textbf{\twitterbot} & \textbf{\ieeecis} & \textbf{\homecredit}\\
    \midrule
    \textit{Clean baseline} & 0.775 & 0.755 &  0.680 \\
    \textit{Robust baseline} & 0.773 & 0.685 & 0.556 \\
    \textit{Random baseline} & 0.566 & 0.500 & 0.501 \\
    \end{tabular}\\
\label{tab:baselines}
\end{table}

\begin{figure*}[t]
    \centering
    \begin{subfigure}[t]{\linewidth}
    \hspace{1.45em}
    \includegraphics[width=0.36\linewidth]{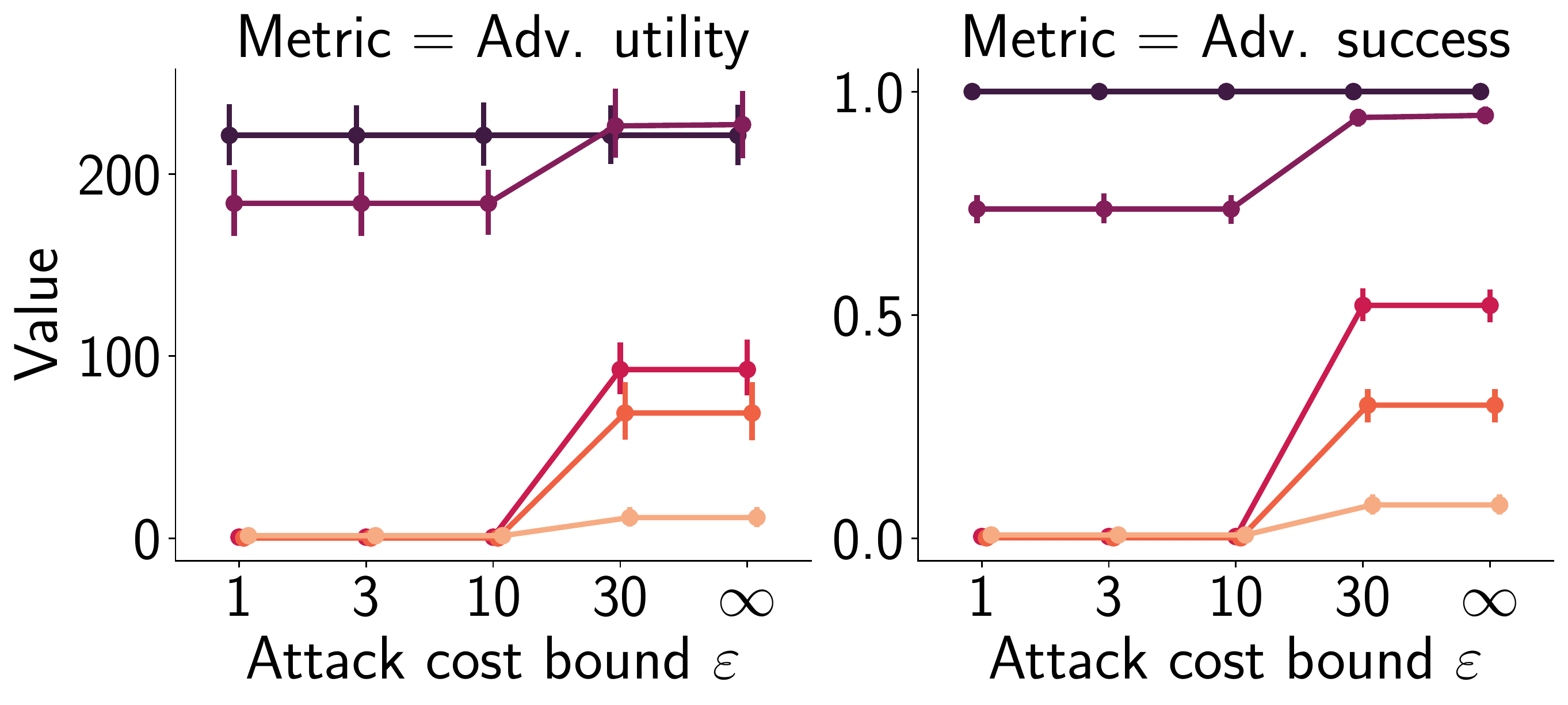}
    \hspace{1.5em}
    \includegraphics[width=0.56\linewidth]{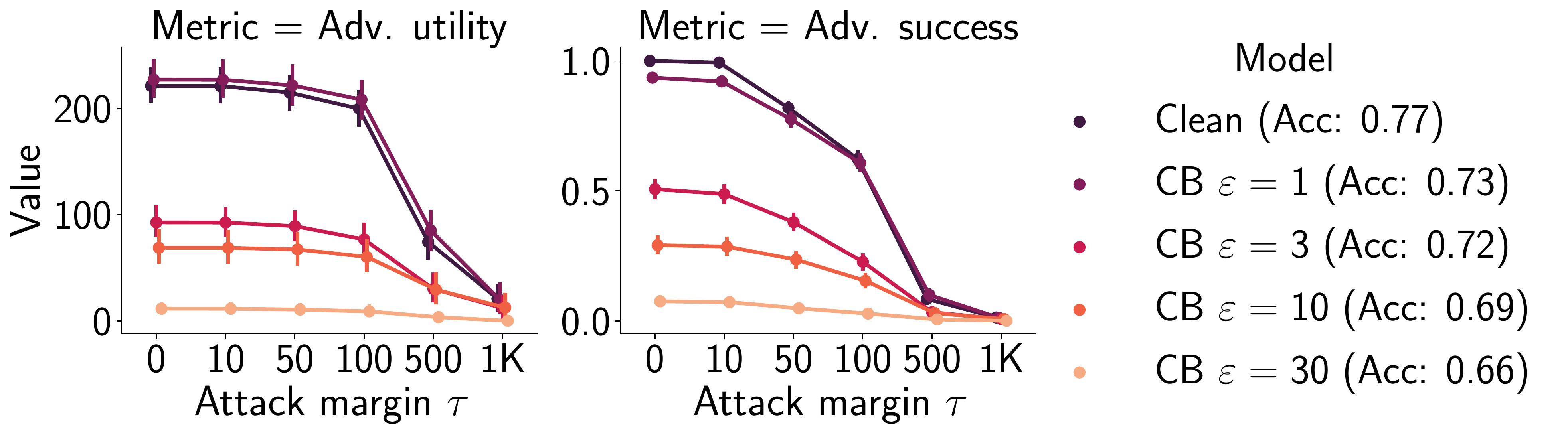}
    \caption{\ieeecis}
    \vspace{1em}
    \end{subfigure}
    \begin{subfigure}[t]{\linewidth}
    \includegraphics[width=0.41\linewidth]{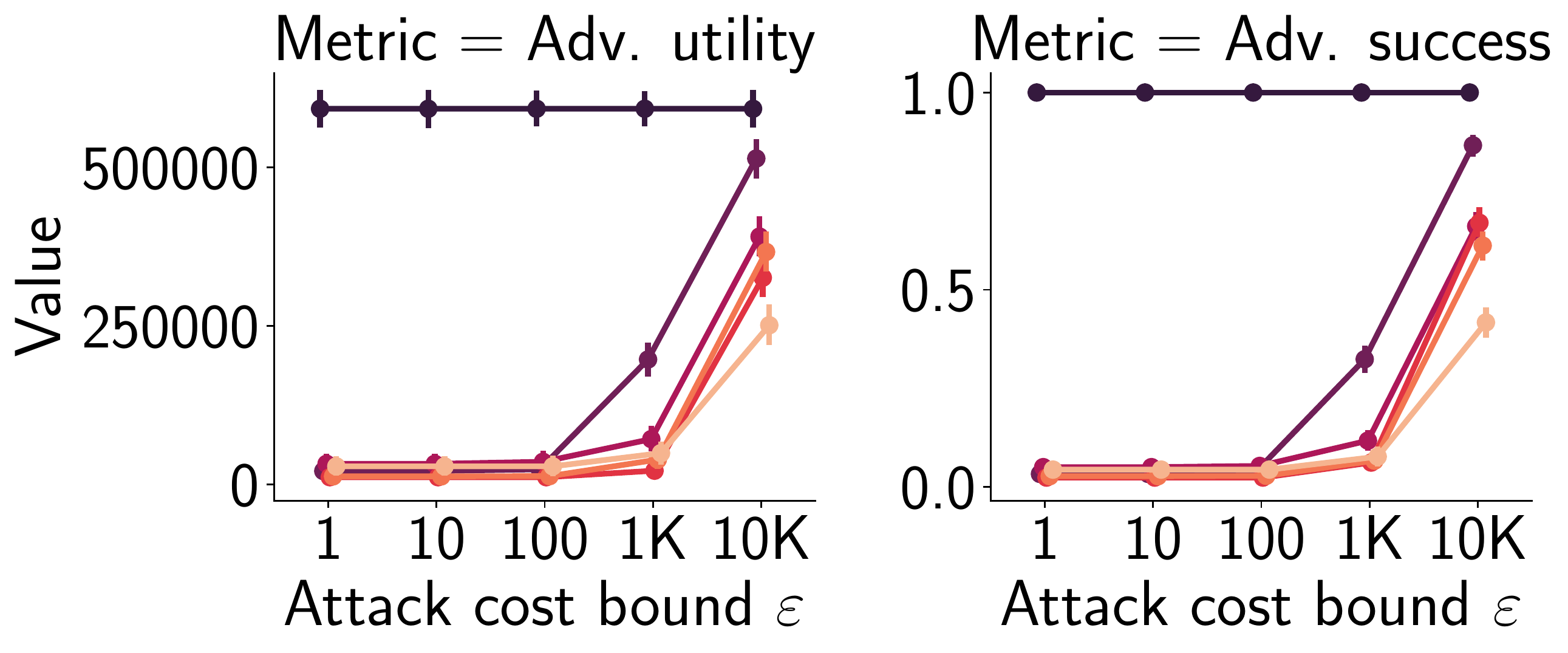}
    ~
    \includegraphics[width=0.59\linewidth]{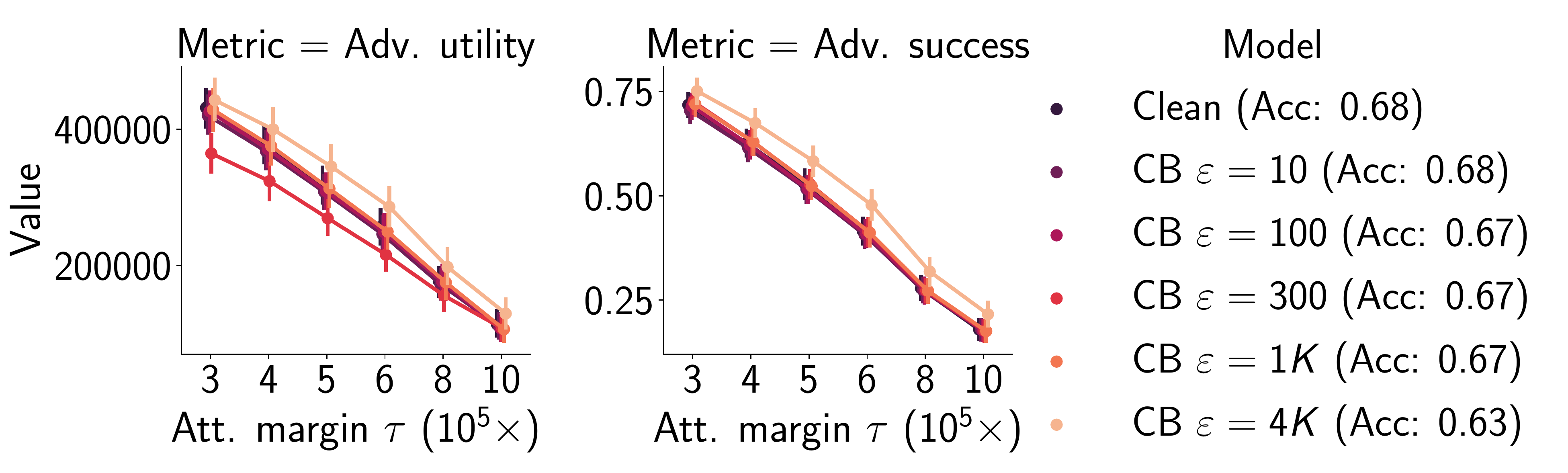}
    \caption{\homecredit}
    \end{subfigure}
    \caption{\textit{Cost-bounded adversarial training} for different adversarial budgets $\varepsilon$. Evaluation against cost-bounded (left), and utility-bounded (right) attacks. We represent the adversary's success and utility (y-axis) versus the adversary's attack budget $\varepsilon$ or desired utility margin $\tau$ (x-axis).
    CB attacks only have substantial success and profit when the adversary invests more than the budget assumed by the defender. UB attacks are thwarted for \ieeecis, but CB training is not significantly effective on \homecredit, and for some models even enables higher adversary's utility.
    }
    \label{fig:defense-cb}
\end{figure*}\begin{figure*}[t]
    \centering
    \begin{subfigure}[t]{\linewidth}
    \hspace{1.45em}\includegraphics[width=0.37\linewidth]{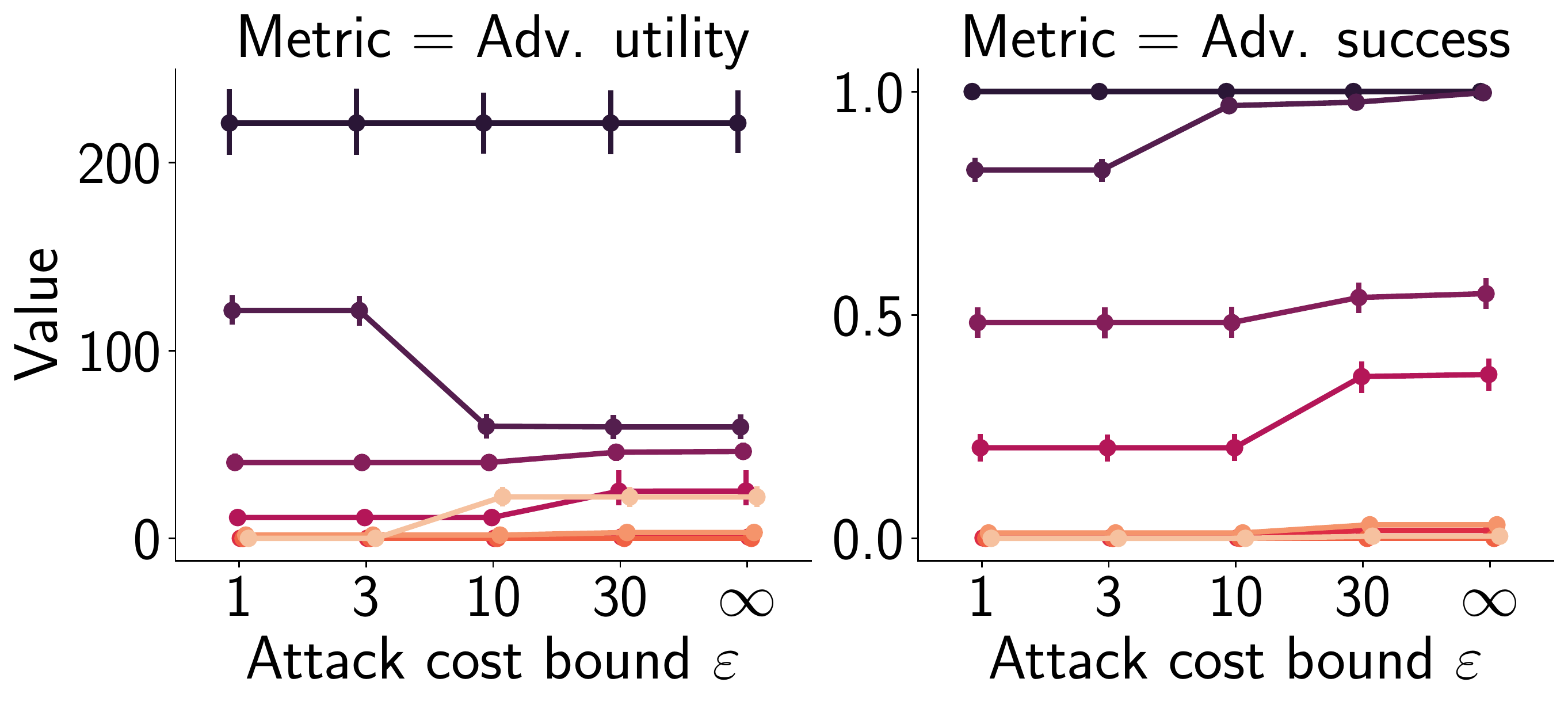}
    \hspace{1.45em}
    \raisebox{-1.3em}{\includegraphics[width=0.58\linewidth]{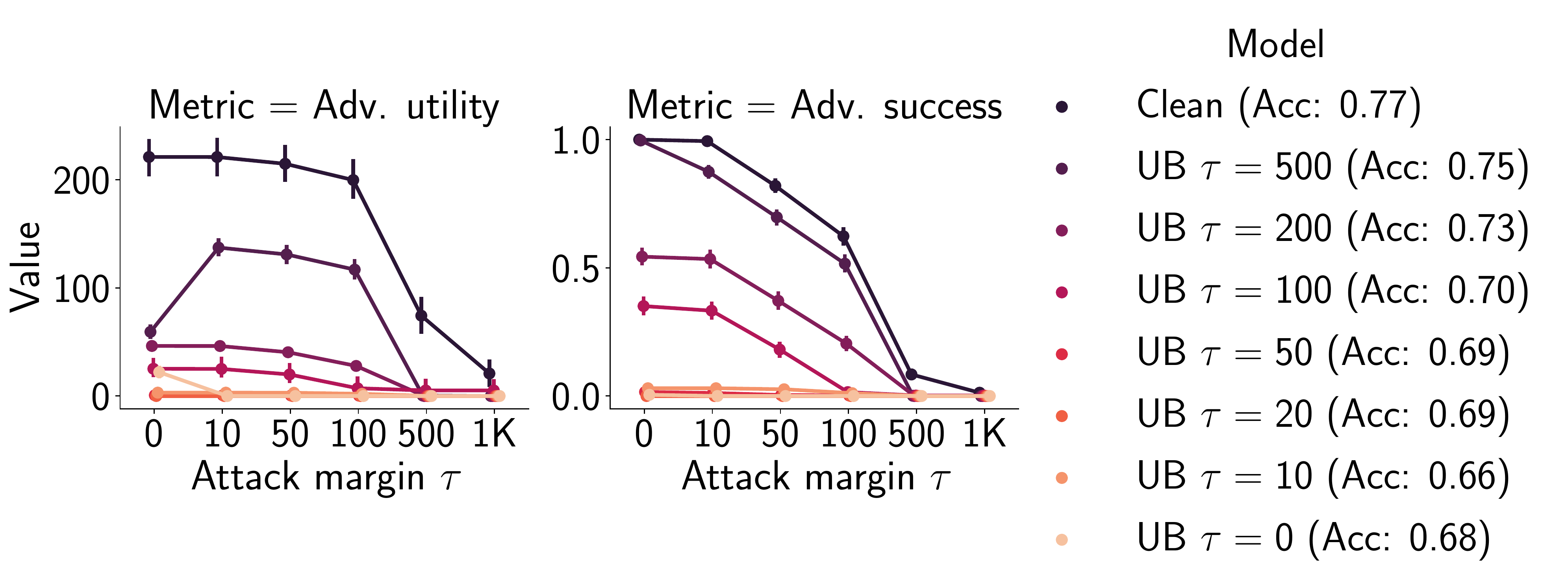}}
    \vspace{-1em}
    \caption{\ieeecis}
    \vspace{1em}
    \end{subfigure}
    \begin{subfigure}[t]{\linewidth}
    \includegraphics[width=0.405\linewidth]{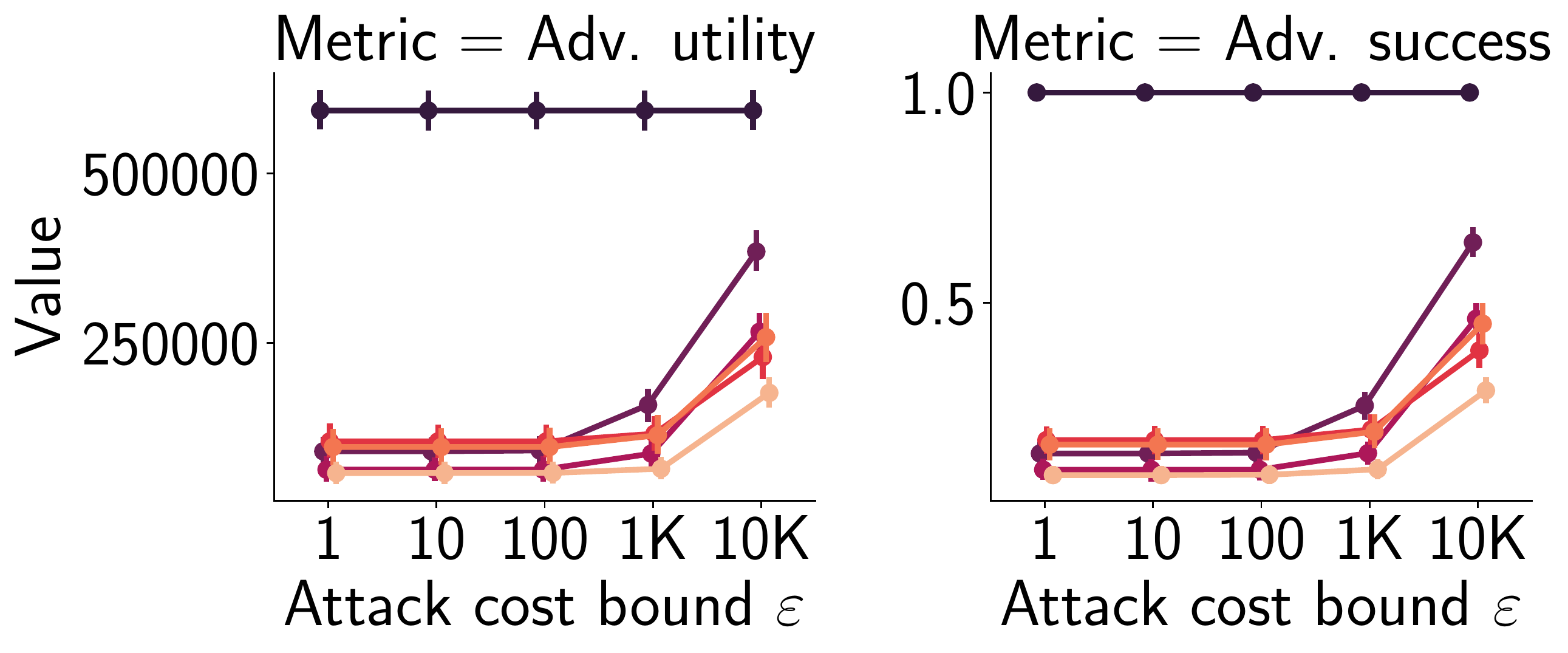}
    ~
    \includegraphics[width=0.595\linewidth]{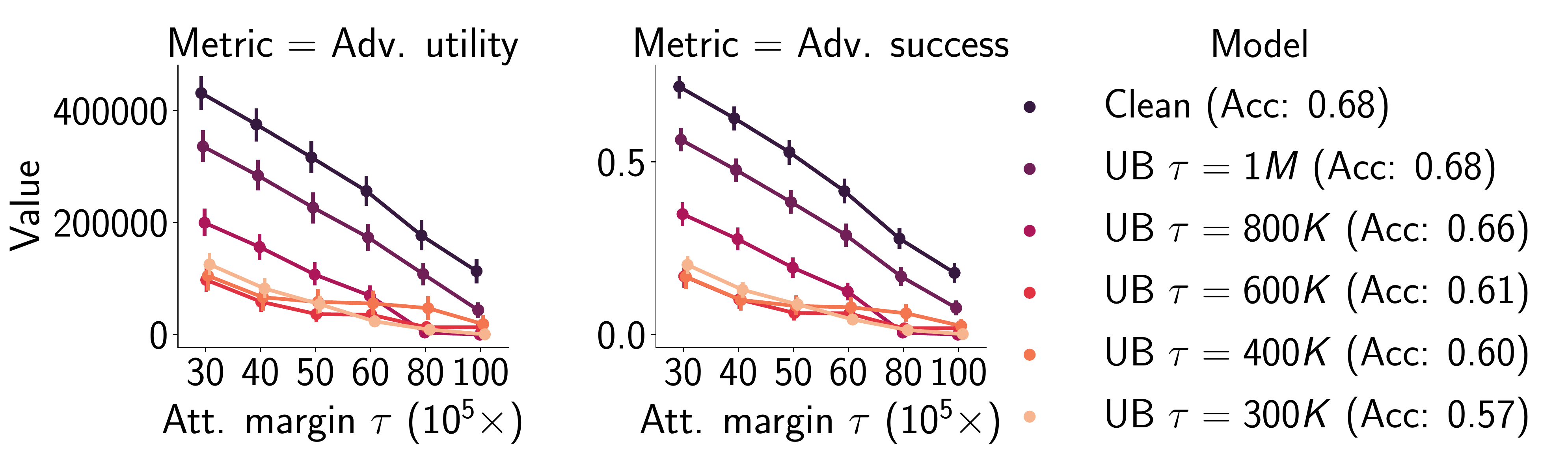}
    \caption{\homecredit}
    \end{subfigure}
    \caption{\textit{Utility-Bounded adversarial training}  for different adversarial utility margins $\tau$. Evaluation against cost-bounded (left) and utility-bounded (right) adversaries.  We show the adversary's success and utility (y-axis) versus the adversary's attack budget $\epsilon$ or desired margin $\tau$ (x-axis).
    On \homecredit, the UB training decreases the performance of both UB and CB attacks, being robustness better against the former. Even when enabling a large profit margin ($\tau=1M$) the attack success rate decreases by 40\%, at the same time not affecting the accuracy.} 
    \label{fig:defense-ub}
\end{figure*}

\subsubsection{Defender Matches the Adversary}
We first evaluate the case in which the adversarial training used to generate the defense is perfectly tailored to the adversary's objective. 

\paragraph{Cost-Bounded Defense vs. Cost-Bounded Attack.}
\cref{fig:defense-cb} shows the results when the defender and the adversary use CB objectives. 
For both \ieeecis and \homecredit, the CB-trained defense is effective when the adversary uses CB attacks:
the adversary only finds successful adversarial examples with positive utility if they invest more than the budget assumed by the defender. If the defender greatly underestimates the adversary's budget of the adversary (e.g., training with $\varepsilon=10$ when the adversary's budget is $\varepsilon=1000$), the adversary obtains a high profit. Therefore, an effective defense requires an adequate estimation of the adversary's capabilities.

\paragraph{Utility-Bounded Defense vs. Utility-Bounded Attack.}
\cref{fig:defense-ub} shows the results of our evaluation when the defender and the adversary use UB objectives. The defense is effective: it decreases both the success rate and the adversary's utility on both datasets. On \ieeecis, the adversary can only succeed when their desired profit $\tau$ is smaller than the $\tau$ used to train the defense. On \homecredit, we observe a similar behaviour, although when training for margins $\tau$ less than 500K, the model does not completely mitigate adversaries that wish to have larger profits. When the defender allows for large adversary's profit margins (e.g., $\tau=800K$ or $\tau=1M$), the models become significantly robust with little accuracy loss.

\subsubsection{Defender Does not Match the Adversary}
In the previous section, we show that if the defender correctly models the attacker's objective, our defenses offer good robustness. Next, we evaluate the performance when the defender's model does not match the adversary's objective. This is likely to happen in realistic deployments, as the defender might not have any a priori knowledge of the adversary's objective.

\paragraph{Utility-Bounded Defense vs. Cost-Bounded Attack.}
\cref{fig:defense-ub} shows our evaluation results when a CB adversary attacks a defense trained assuming UB objectives. For both datasets, the robustness improves with respect to the clean baseline, even though robustness against CB adversaries is not the defense goal. The improvement is more pronounced as the defender tightens the profit margin (decrease in $\tau$). This effect is stronger on \homecredit, where even loose profit margins provide significant robustness. The adversary can increase their success (on both datasets) and utility (only \homecredit) by increasing their budget $\epsilon$. These experiments show that UB training improves robustness \textit{even when the adversary has a different objective}.

\paragraph{Cost-Bounded Defense vs. Utility-Bounded Attack.}
When a CB defense confronts a UB adversary, we observe a different behaviour, showed in \cref{fig:defense-cb}. On \ieeecis, CB adversarial training increases the robustness of the model against UB adversaries, with greater effect as the cost bound increases. When protecting against high adversary's budgets ($\varepsilon=30$), however, the impact on accuracy is too large, and the robust baseline becomes preferable.
For \homecredit, the situation is worse. Although performance is always above the robust baseline, we observe little improvement with respect to the clean model. Even worse, for certain parameters, the utility of the adversary can increase after adversarial training (see the model trained with a bound of $\varepsilon=4000$).
We conclude that CB training is not effective against an adversary with a different objective.

\subsubsection{Robustness-Accuracy Trade-offs}
In the previous sections, we evaluate the effectiveness of the defenses depending on the adversary's and defender's objectives. We now evaluate the trade-offs between defense effectiveness in reducing the adversary's performance on one hand, and the accuracy of the model on the other hand.

As adversarial training penalizes model's sensitivity to modifications of input features, it results in certain features having less influence on the output. These features cannot be used for prediction to the same extent as features in the clean baseline, which leads to the degradation of the model's accuracy. On the positive side, these features can neither be used by the adversary---the robust baseline being extreme in which all features prone to manipulation are zeroed---reducing the attack's success and utility.

In our experiments, we observe different nature of these trade-offs for robustness against CB and UB adversaries. Against the CB adversary, the accuracy-robustness trade-off depends on the modeled adversary's budget, and we do not observe that either defense approach is superior to another. On the contrary, against the UB adversaries, we consistently observe better robustness (less adversarial utility for the same accuracy) for the UB defense compared to the CB one. We present the detailed accuracy-robustness plots in \cref{fig:defense-tradeoffs} in \cref{sec:experiments-extra}.

Thus, in the absence of knowledge of the adversary's objective, utility-bounded defenses are preferable. They outperform CB adversarial training when the adversary is utility-oriented, and offer comparable performance against CB attacks.

\section{Related Work}
\label{sec:related}

Our conceptual contributions span three aspects of adversarial robustness in tabular domains: new formulations of (A) \emph{adversarial objectives}, (B) \emph{attack strategies} within these objectives, and (C) \emph{adversarial training}-based defenses. We review the related work in each of the aspects next. We also provide a concise summary in \cref{tab:related-objectives}.

\begin{table*}[t]
\caption{Summary of related work in terms of three aspects: adversarial models, attack strategies, and defense strategies. Adversarial models: \textit{Adv. cost} --- description of (an equivalent of) an adversarial cost model. \textit{Adv. utility} --- whether adversarial gain is incorporated into the model. Attacks: \textit{Targets} --- which target models can be attacked. \textit{Feasibility} --- whether the attack is guaranteed to produce a feasible adversarial example. \textit{Algorithm} --- a short description of the algorithmic approach. Defenses: \textit{Arch.} --- which model architectures are supported.}
\label{tab:related-objectives}
\resizebox{\linewidth}{!}{
\begin{tabular}{l|ll|lll|l}
\toprule
& Adversarial models & & Attacks & & & Defenses \\
Adversarial Model                                & Adv. cost & Adv. utility & Targets & Feasibility & Algorithm & Arch. \\
\midrule
\citet{ballet2019imperceptible} & Feature-importance based  & --- & Differentiable  & \xmark & Gradient-based & --- \\
\citet{CartellaAFDAE21}         & Feature-importance based  & --- & \textbf{Any}    & \xmark & ZOO & ---  \\
\citet{levy2020not}             & Distance based    & --- & \textbf{Any}    & \xmark & Gradient-based & --- \\
\citet{KantchelianTJ16}         & $L_p$                      & --- & Tree-based     & \cmark & MILP & --- \\
\citet{AndriushchenkoH19}       & $L_\infty$                 & --- & Tree-based     & \cmark & Custom & Tree-based \\
\citet{ChenWJCJ21}              & Per-feature constraints   & --- & ---      & --- & --- & Tree-based \\
\citet{CalzavaraLTAO20}         & Per-feature constraints   & --- & Tree-based      & \cmark & Exhaustive search & Tree-based \\
\citet{VosV21}                  & Per-feature constraints   & --- & ---      & --- & --- & Tree-based \\ 
Ours                            & \textbf{Generic} (\cref{sec:assumptions}) & \cmark & \textbf{Any} & \cmark & Graph search & Differentiable \\
\bottomrule
\end{tabular}
}
\end{table*}

\subsection{Adversarial Objectives}
In this part, we review the related adversarial objectives as well as some approaches which are similar in spirit to our adversarial objectives. 

\paragraph{Cost-Based Objectives.}
\label{sec:l_inf}
Our generic cost-bounded objective is not the only possible approach to model attacks in tabular domains.
For example, works on adversarial robustness in the context of decision tree-based classifiers often use per-feature constraints as adversarial constraints~\cite{ChenWJCJ21, AndriushchenkoH19, ChenZBH19}. At the low level, these constraints are formalized either as bounds on $L_\infty$ distance~\cite{AndriushchenkoH19, ChenZBH19}, or using functions determining constraints for each feature value~\cite{ChenWJCJ21}. In these approaches, the feature constraints are independent. Such independence simplifies the problem. For example, the usage of $L_\infty$ constraints enables to split a multidimensional optimization problem into a combination of simple one-dimension tasks~\cite{AndriushchenkoH19}, or to limit the set of points affected by the split change \cite{ChenWJCJ21}. %
Unfortunately, per-feature constraints cannot realistically capture the \emph{total cost} of mounting an attack: the aggregate cost of all the feature modifications required to produce an adversarial example, which is crucial to capture in tabular domains.

Also related to our cost-based proposal, \citet{Pierazzi2020IntriguingPO} introduce a general framework for defining attack constraints in the \textit{problem space}. Our cost-based objective can be thought as an instance of this framework: we encode the problem-space constraints in the set of feasible examples.

Our cost model resembles the Gower distance~\cite{gower1971general}, which is also a sum of ``dissimilarities'' across different categories of features. As opposed to this distance, our cost model can accommodate a wider class of numeric features, e.g., with a non-linear cost of changes. Also, it is not bounded to $[0,1]$ interval providing flexibility to model a wider range of applications. 

\paragraph{Utility-Based Objectives.} The literature on \newterm{strategic classification} also considers utility-oriented objectives~\cite{HardtMPW16, DongRSWW18, Milli19} for their agents. In this body of work, however, agents are not considered adversaries, and the gain is typically limited to \{$+1$,$-1$\} reflecting the classifier decision. Our model supports arbitrary gain values, which enables us to model broader interests of the adversary such as revenue. Only the work by \citet{SundaramVXY21} supports gains different from $+1$ or $-1$, but they focus on PAC-learning guarantees in the case of linear classifiers, whereas our goal is to provide practical attack and defense algorithms for a wider family of classifiers.

\subsection{Attack Strategies}

\paragraph{Tabular Domains. }
Several works have proposed attacks on tabular data. 
\citet{ballet2019imperceptible} and \citet{CartellaAFDAE21} propose to apply existing continuous attacks to tabular datasets. The authors focus on crafting imperceptible adversarial examples using standard methods from the image domain. They adapt these methods such that less ``important'' features (low correlation with the target variable) can be perturbed to a higher degree than other features. This corresponds to a special case within our framework, in which the feature-modification costs depend on the feature importance, with the difference that these approaches cannot guarantee that the proposed example will be feasible. 
\citet{levy2020not} propose to construct a surrogate model capable of mimicking the target classifier. A part of this surrogate model is a feature-embedding function which maps tabular data points to a homogeneous continuous domain. They apply projected gradient descent to produce adversarial examples in the embedding space and map the resulting examples back to the tabular domain. As opposed to our methods, Levy et al. cannot provide any guarantee that the produced adversarial example lay in the feasible set.
Finally, \citet{KantchelianTJ16} propose a MILP-based attack and its relaxation within different $\lp{p}$ cost models against random-forest models.
Our attack differs from these three methods as they use $\lp{p}$ or similar bounds, whereas we use a cost bound that can capture realistic constraints as explained in \cref{sec:intro,sec:statement}.

\paragraph{Text Domains. } Our universal greedy attack algorithm is similar to the methods for attacking classifiers that operate on text~\cite{ZhangSAL19, YangCHWJ20, WangHBSMLZ20, WangPLL19, LeiWCDDW18, EbrahimiRLD18, LiangLSBLS18}. %
All these works, however, make use of adversarial constraints such as restrictions on the number of modified words or sentences. These constraints do not apply to tabular domains, as simply considering ``number of changes'' does not address the heterogeneity of features. Our algorithms also differ from these approaches in that we incorporate complex adversarial costs in the design of the algorithms.
For example, the Greedy attack by \citet{YangCHWJ20} uses the target classifier's confidence for choosing the best modifications to create adversarial examples while accounting for the number of modifications. Our framework not only considers the number of modifications but also their cost, thus capturing richer constraints of the adversary.

\subsection{Adversarial Training}
We discuss existing defense methods and techniques with related goals, which appear in the context of decision tree-based models. Adversarial robustness of such classifiers has been studied extensively~\cite{ChenZBH19, AndriushchenkoH19, CalzavaraLTAO20, ChenWJCJ21, VosV21}. These works assume independent per-feature adversarial constraints, e.g., based on the $L_\infty$ metric. Our adversarial models, and thus our attacks and defenses, are capable of capturing a broader class of adversarial cost functions that depend on feature modifications and better model the adversary's constraints as we explain in \cref{sec:l_inf}.

\section{Concluding Remarks and Future Work}
In this paper, we have revisited the problem of adversarial robustness when the target machine-learning model operates on tabular data. We showed that previous approaches, tailored to produce adversarial image or text examples, and defend from them, perform poorly when used in tabular domains. This is because they are conceived within a threat model that does not capture the capabilities and goals of the tabular adversaries.

We introduced a new framework to design attacks and defenses that account for the constraints existing in tabular adversarial scenarios: adversaries are limited by a budget to modify features, and adversaries can assign different utility to different examples. Having evaluated these attacks and defenses on three realistic datasets, we showed that our novel utility-based defense not only generates models which are robust against utility-aware adversaries, but also against cost-bounded adversaries. On the contrary, performing adversarial training considering a cost-bounded adversary---as traditionally done in the literature---is a poor defense against adversaries focused on utility in some scenarios.

Although our high-level cost and utility-based framework is designed for tabular data, it can be applied to any case where possible adversarial actions can be modeled using generic costs, and where different attacks can have a different value for the adversary. As an example, our adversarial model could be applied to defending against attacks on a visual traffic sign recognition~\cite{sitawarin2018darts}, where the attack harms could vary wildly from the misclassification of one type of road sign to another.

The main limitation of our work is that the effectiveness of the defenses relies on strong assumptions. The defender has to not only correctly model the adversary's objective, but also to estimate the adversary's cost model, budget, and how much value they ascribe to the success of each attack. More work is needed to understand the implications of misspecifications of this model, and to design algorithms that can provide protection even under misspecifications.

\section*{Acknowledgements}
This work was partially funded by the Swiss National Science Foundation with grant 200021-188824. The authors would like to thank Maksym Andriushchenko for his helpful feedback and discussions.

\bibliographystyle{unsrtnat}
\bibliography{main}

\appendix
\crefalias{section}{appendix}
\section{Other Possible Adversarial Objectives}\label{app:extra-discussion}

We propose a cost-oriented and a utility-oriented adversarial objective in \cref{sec:statement}. These are not the only possible formalizations for our high-level goals. One other approach is an adversary maximizing utility subject to a cost budget:
\begin{equation}\label{eq:max-util-ct}
    \begin{aligned}
    \max_{\x \in \feasible(\initx, y)} & \id[\decision(\x) \neq y] \cdot u_{\x, y}(\x') \\
        & = \id[\decision(\x) \neq y] \cdot \left[ \gain(\x') - \cost(\initx, \x') \right]_+ \\
        & \text{s.t. } \cost(\initx, \x') < \varepsilon
    \end{aligned}
\end{equation}

This formalization is a middle ground between our cost-constrained and utility-constrained objectives: On the one hand, the adversary is aware of the utility of a given example. On the other hand, they do not adjust their budget for different examples, i.e. the constraint for \$10 and \$1,000 stays the same, even though the adversary clearly differentiates in their value. We conducted preliminary experiments with this objective, and its results are marginally different from the cost-bounded one in our experimental setup.

\newcommand{\highlightcolor}{blue!70}
\section{Details on Adversarial Training}
\label{sec:proj-details}
We describe our modifications to the traditional adversarial training pipeline.
Our training procedure is a version of the well-known adversarial training algorithm based on the PGD method~\cite{MadryMSTV17}. 
We design an adapted projection algorithm to solve~\cref{eq:cost-constrained-hard-defense-relax}, presented in \cref{algo:cost-proj}. This algorithm is an extension of an existing sort-based weighted $\lp{1}$ projection algorithm~\cite{slav2010ieee,perez2020efficient}. It takes as input a sample $\rx$ and a perturbed sample $\rx'$, and returns a valid perturbation vector $\delta$ such that $\rx + \delta$ lies within the cost budget. Compared to the algorithm by \citet{perez2020efficient}, we introduce the capability to assign different weights based on the feature type and perturbation sign (line 2, in \textcolor{\highlightcolor}{blue}) to support our cost function in~\cref{eq:final-cost-function}.

\begin{algorithm}[t]
    \caption{Cost-bounded Adversarial Training Algorithm (single iteration)}
    \label{algo:pgd-train}
	\textbf{Input:} Model weights $\theta$, batch of training examples $(\enc({\x}^{(i)}), y^{(i)})_{i=1}^{b}$, 
	per-feature costs $w_i$, cost bound $\varepsilon$, number of PGD steps $t$. \\
	\textbf{Output:} Updated weights $\theta'$
    \begin{algorithmic}[1]
        \State $\alpha := 2 \frac{\varepsilon}{n}$
        \For{ $i \textbf{ in } 1, \ldots, b$}
        \State $\delta^{(i)} := 0$
        \For{ $t \textbf{ in } 1, \ldots, n$}
        \State $g^{(i)} := \nabla_{\delta_i} \ell(f_{\theta}(\enc({\x}^{(i)})+\delta^{(i)}), y_i)$
        \State $\delta^{(i)} := \delta^{(i)} + \alpha \frac{g^{(i)}}{\|g^{(i)}\|_1}$
        \State $\delta^{(i)} := P_{\tilde B_\varepsilon(\x^{(i)}, y^{(i)})}(\enc({\x}^{(i)}) + \delta^{(i)})$
        \EndFor
        \EndFor
    \State $\theta' := \theta - \eta \nabla_{\theta} \frac{1}{b} \sum_{i=1}^b \ell(f_{\theta}(\enc({\x}^{(i)})+\delta^{(i)}), y^{(i)})$
    \end{algorithmic}
    \hspace*{\algorithmicindent} \textbf{Return} $\theta'$
\end{algorithm}

We now prove the correctness of this algorithm.
\begin{statement}\label{stmt:proj-correctness}
\cref{algo:cost-proj} is a valid projection algorithm onto the set $\tilde B_\varepsilon$, as defined in \cref{eq:convex-superset}. For a given $\rx, \rx'$, the algorithm returns $\delta^*$ such that:
\begin{equation}
\label{proj-min}
\begin{aligned}
    \delta^* = P_{\tilde B_\varepsilon(\x, y)}(\rx') & \define \argmin_{\delta \in \sR^d}\ \|\rx' - (\rx + \delta)\|_2 \\
    & \text{s.t. } \rcost(\rx, \rx + \delta) \leq \varepsilon
\end{aligned}
\end{equation}
\end{statement}
\begin{proof}
Observe that if we keep either $c_{j+}({\x})$ or $c_{j-}({\x})$, the constraint becomes a standard weighted $\lp{1}$ constraint. Projection onto a weighted $\lp{1}$ ball is equivalent to projection onto the \newterm{simplex}~\cite{perez2020efficient}
in non-trivial cases in which the projection point lies outside of the ball. The algorithm by \citet{perez2020efficient} is a valid projection onto the simplex. We want to prove that our optimization can be transformed into projection onto the simplex, which means we can use their algorithm. In our notation, the simplex can be defined as the following set:
\begin{equation}\label{eq:simplex}
    \{ \rx' \in \sR^d \mid \rcost^*(\rx, \rx') = \varepsilon \}, \text{ with } \rcost^*(\rx, \rx') \define \sum_{j = 1}^{d} w_j \, | \rx_j - \rx'_j|,
\end{equation}
where the weights $w_j > 0$ are positive constants for $i \in [d]$, and $\rx_j$ denotes the $j$-th dimension of an encoded vector $\rx = \enc(\x)$.

In order to find the $w_i$ weights of our target simplex, we first show that the projection onto our cost ball lies in the same quadrant as the original point:

\begin{lemma}
\label{app-lemma}
For any $\rx, \rx', \varepsilon$, and $\delta^*$ defined as follows: 
\begin{equation}
\delta^* = \argmin_{\delta \in \sR^d:\ \rcost(\rx, \rx + \delta) \leq \varepsilon} \|\rx' - \rx - \delta\|_2.
\end{equation}
It holds for any $i \in [d]$ that either $\text{sign}(\delta_i^*) = \text{sign}(\rx'_i - \rx_i)$, or $\text{sign}(\delta_i^*) = 0$.
\end{lemma}

\begin{proof}
Proof by contradiction. Let us assume that the lemma does not hold. Then, there exist $i$ such that $\text{sign}(\delta^*_i) = -\text{sign}(\rx'_i - \rx_i) \text{ and } \text{sign}(\delta^*_i) \neq 0$. Then, we can construct $\delta^+$ such that for all $\forall j \neq i$, $\delta^+_j = \delta^*_j \text{ and } \delta^+_i = -\delta^*_i$. 
$$
\|\rx' - \rx - \delta^*\|^2_2 = \|\rx' - \rx - \delta^+\|^2_2 - (\rx'_i - \rx_i - \delta^*_i)^2 + (\rx'_i - \rx_i - \delta^+_i)^2 
$$
Since $\text{sign}(\delta^*_i) = -\text{sign}(\rx'_i - \rx_i) \text{ and } \text{sign}(\delta^+_i) \neq 0$, 
$$
(\rx'_i - \rx_i - \delta^*_i)^2 > (\rx'_i - \rx_i - \delta^+_i)^2
$$
Therefore,
$$
\|\rx' - \rx - \delta^+\|^2_2 < \|\rx' - \rx - \delta^*\|^2_2,
$$
which is a contradiction to the original statement.
\end{proof}
Based on this lemma we can see that in order to find the projection of $\rx'$, we can replace $\rcost(\rx, \rx')$ with the following expression:
\begin{equation} 
\rcost^*(\rx, \rx') = \sum_{i \in \mathcal{C}} \frac{1}{2} \|\bar \x_i - \bar \x'_i\|_1 \min_{t \in \feasible_i(x, y)} \cost_i(\x_i, \x_i') + \sum_{j \in \mathcal{I}} c_{j*}(\x) \cdot |\bar \x_j' - \bar \x_j|, \\
\end{equation}
where $c_{j*}$ is defined as follows:
\[
c_{j*}({\x}) = \begin{cases} 
c_{j+}({\x}), & \text{if } \text{sign}(\rx'_j - \rx_j) \geq 0\\
c_{j-}({\x}), & \text{if } \text{sign}(\rx'_j - \rx_j) < 0 \\
\end{cases}
\]
Observe that this expression is an instance of a simplex constraint in \cref{eq:simplex}. Therefore, once we have set the appropriate weights $w_i$ in Line 2 of \cref{algo:cost-proj}, the rest of the algorithm is a valid projection.
\end{proof}

\begin{figure}[t]
\vspace{-0.12in}
\begin{algorithm}[H]
    \caption{Cost-Ball Projection Algorithm}
    \label{algo:cost-proj}
	\hspace*{\algorithmicindent} \textbf{Input}  $\rx, \rx', \cost, \varepsilon, \mathcal{C}, \mathcal{I}$ \\
	\hspace*{\algorithmicindent} \textbf{Output} $\delta^* = P_{\tilde B_\varepsilon(\x, y)}(\rx')$
    \begin{algorithmic}[1]
        \State $\delta = \rx' - \rx$
	    \State \textcolor{\highlightcolor}{$w_i := \begin{cases}
                        \min_{t \in \feasible_i(x, y)} \cost_i({\x}_i, t),& \text{if } i \in \mathcal{C}\\
                        c_{j-}({\x}),  & \text{if } i \in \mathcal{I} \text{ and } \delta_i < 0 \\
                        c_{j+}({\x}),  & \text{if } i \in \mathcal{I} \text{ and } \delta_i \geq 0 \\
                        \end{cases}$}
        \State $z_i := \frac{\delta_{i}}{w_{i}}$
        \State $\pi_z() := \text{Permutation } \uparrow (z)$
        \State $z_i := z_{\pi_z(i)}$
        \State $J := \max \left\{ j: \frac{-\varepsilon + \sum_{i=j+1}^{m} w_{\pi_z(i)} \delta_{\pi_z(i)}}{\sum_{i=j+1}^{m} w_{\pi_z(i)}^{2}}>z_{j} \right\}$
        \State $\lambda := \frac{-\varepsilon+\sum_{j=J+1}^{m} w_{\pi_z(j)} \delta_{\pi_z(j)}}{\sum_{j=J+1}^{m} w_{\pi_z(j)}^{2}}$
        \State $\delta^*_i := \text{sign}(\delta_i) \max \left(\delta_{i}-w_{i} \lambda, 0\right)$
    \end{algorithmic}
    \hspace*{\algorithmicindent} \textbf{Return} $\delta^*_i$
\end{algorithm}
\vspace{-0.1in}
\footnotesize
\setlength\fboxsep{0pt}
The \textcolor{\highlightcolor}{highlighted} parts indicate the differences with respect to the sort-based weighted $\lp{1}$ projection algorithm~\cite{perez2020efficient}. The function $\pi_z(i)$ denotes an outcome of permutation. $\text{Permutation } \uparrow (z)$ is a sort permutation in an ascending order.
\end{figure}

\section{Additional Details on the Experiments}
\label{sec:experiments-extra}
We provide the details of our experimental setup. 

\begin{algorithm}[t]
    \caption{PGD-Based Attack}
    \label{algo:pgd-alg}
	\textbf{Input:} Initial example $\rx$, label $y$, costs $w$, cost bound $\varepsilon$. \\
	\textbf{Output:} Adversarial example $\rx'$
    \begin{algorithmic}[1]
        \State $\alpha := 2 \frac{\varepsilon}{n}$
        \State $\delta := 0$
        \For{ $j \textbf{ in } 1..n$}
            \State $\nabla := \nabla_{\delta} \ell(f_{\theta}(\rx+\delta), y)$
            \State $\delta := \delta + \alpha \frac{\nabla}{\|\nabla\|_1}$
            \State $\delta^* = P_{\tilde B_\varepsilon(\x, y)}(\rx + \delta)$
        \EndFor
        $\rx' = \rx' + \delta^*$
    \end{algorithmic}
    \hspace*{\algorithmicindent} \textbf{Return} $\rx'$
\end{algorithm}

\begin{figure}[t]
    \centering
    \includegraphics[width=.33\linewidth]{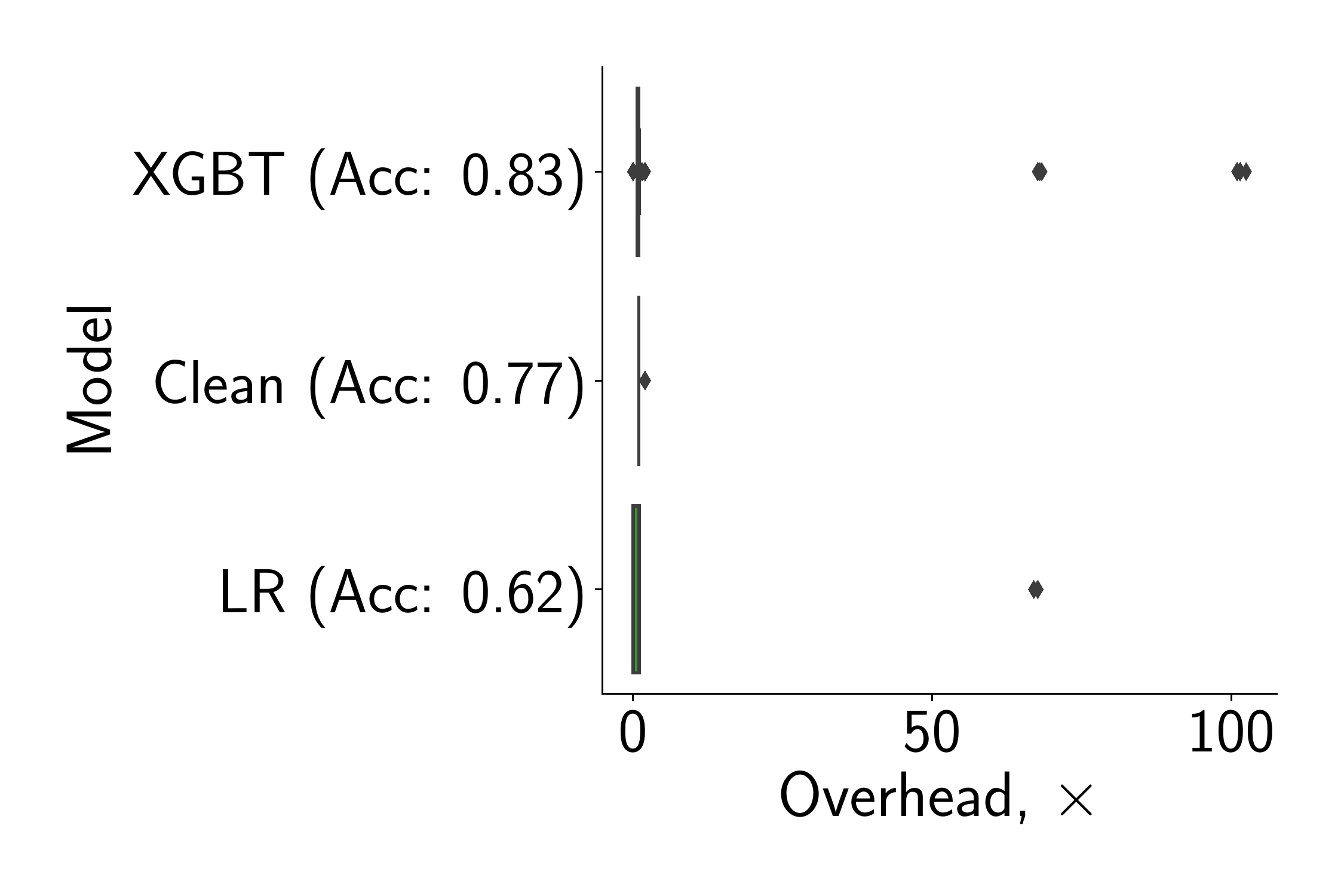}
    \caption{The distribution of cost overhead of adversarial examples obtained with UG over minimal-cost adversarial examples obtained with UCS on \ieeecis. Most UG adversarial examples have cost close to the minimal, although there exist outliers.}
    \label{fig:ieeecis-overhead}
\end{figure}

\label{sec:hyperparameters}

\begin{table}[t]
    \centering
    \caption{\ieeecis and \homecredit attack and defense parameters}
    \begin{tabular}{@{}ll}
    \midrule
    \textbf{Parameter} & \textbf{Value range} \\
    \midrule
    \multicolumn{2}{l}{Adversarial Training (\ieeecis)}  \\
    \midrule
    \textit{Batch size} & 2048 \\
    \textit{Number of epochs} & 400 \\
    \textit{PGD iteration number} & 20 \\
    \textit{TabNet hyperparameters} & $N_{D}=16,N_{A}=16,N_{steps}=4$ \\
    \textit{$\varepsilon$ (for CB models)} & $[1, 3, 10, 30]$ \\
    \textit{$\tau$ (for UB models)} & $[ 0, 10, 20, 50, 100, 200, 500]$ \\
    \midrule
     \multicolumn{2}{l} {Attacks (\ieeecis)} \\
    \midrule
    \textit{Max. iterations} & 100K \\
    \textit{$\varepsilon$ (for CB attacks)} & $[1, 3, 10, 30]$ \\
    \textit{$\tau$ (for UB attacks)} & $[0, 10, 50, 500, 1000]$ \\

    \midrule
    \multicolumn{2}{l}{Adversarial Training (\homecredit)} \\
    \midrule
    \textit{Batch size} & 2048 \\
    \textit{Num. of epochs} & 100 \\
    \textit{TabNet hyperparameters} & $N_{D}=16,N_{A}=16,N_{steps}=4$ \\
    \textit{Num. of PGD iterations} & 20 \\
    \textit{$\varepsilon$ (for CB models)} & $[1, 10, 100, 1000, 10000]$ \\
    \textit{$\tau$ (for UB models)} & $[ 300K, 400K, 500K, 600K, 800K ]$ \\
    \midrule
     \multicolumn{2}{l}{Attacks (\homecredit)} \\
    \midrule
    \textit{Num. of iterations} & 100 \\
    \textit{$\varepsilon$ (for CB attacks)} & $[1, 10, 100, 1K, 10K]$ \\
    \textit{$\tau$ (for UB attacks)} & $[10K, 300K, 400K, 500K, 600K, 800K]$ \\
    \end{tabular}
\label{tab:hpar-homecredit}
\end{table}

\subsection{Hyperparameter selection}\label{sec:app-hyp}
We list our defense and attack parameters in \cref{tab:hpar-homecredit}. The TabNet parameters are denoted according to the original paper~\cite{arik2021tabnet}. We set the virtual batch size to 512. As training the clean baseline for HomeCredit was prone to overfitting in our setup, we reduced the number of training epochs to 100. Other hyperparameters were selected with a grid search.

\subsection{Dataset Processing and Adversarial Cost Models}\label{sec:cost-models}
For each dataset, we create an adversarial cost based on hypothetical scenarios. In this section, we describe how we process the data, and how we assign costs to modifications of the features in each dataset.

\subsubsection{\twitterbot}
We use 19 numeric features from this dataset. We drop three features for which we cannot compute the effect of a transformation as we do not have access to the original tweets. We use the number of followers as the adversary's gain. We assign costs of features based on estimated costs to purchase Twitter accounts of different characteristics on darknet markets.

\subsubsection{\ieeecis}
We ascribe cost of changes, assuming that the adversary can change the device type and email address at a small cost. The device type can be changed with low effort using specific software. Email domain can be changed with a registration of a new email address which typically cannot be automated. Although also low cost, it takes more time and effort than changing the device time. We reflect these assumptions ascribing the costs \$0.1 and \$0.2 to these changes. Changing the type of the payment card requires obtaining a new card, which costs approximately \$20 in US-based darknet marketplaces as of the time of writing. We consider the transaction amount as a gain obtained by an adversary.

\subsubsection{\homecredit}
The main goal of the adversary in this task is receiving a credit approval. As one example represents a loan application, we set the credit amount to be the gain of the example. All features which can be used by an adversary are listed in \cref{tab:homecredit_costs} with the costs we ascribe to them. We assume six groups of features and estimate the cost as follows: 
\begin{itemize}
    \item \textit{Group 1}: Features that an adversary can change with negligible effort such as an email address, weekday, or hour of the loan application. We ascribe \$0.1 cost to these transformations.
    \item \textit{Group 2}: Features associated to income. We use these as numerical features to illustrate the flexibility of our method. We assume that to increase income by \$1, the adversary needs to pay \$1.
    \item \textit{Group 3}: Features associated to changing a phone number. Based on the US darknet marketplace prices as of the time of writing, we estimate that purchasing a SIM card costs \$10.
    \item \textit{Group 4}: Features related to official documents which can be temporally changed. For example, a car's ownership can be transferred from one person to another for the application period, and returned to the original owner after it. We ascribe a cost of \$100 to these changes. 
    \item \textit{Group 5}: Features that require either document forging or permanent changes to a person's status. For instance, purchasing a fake university diploma. For the sake of an example, we estimate their cost at \$1,000.
    \item \textit{Group 6}: Features related to credit scores provided by external credit-scoring agencies. We estimate the cost of changes in this group with a manipulation model based on a real-world phenomenon of credit piggybacking, described next.
\end{itemize}

\begin{table}[t!]
    \centering
    \small
    \caption{Costs of changing a feature in \twitterbot dataset}
    \label{tab:twitter_costs}
    \begin{tabular}{@{}lc}
    \textbf{Feature} & \textbf{Estimated cost, \$} \\
    \midrule 
    \textit{likes\_per\_tweet} & 0.025 \\
    \textit{retweets\_per\_tweet} & 0.025 \\
    \textit{user\_tweeted} & 2 \\
    \textit{user\_replied} & 2 \\
    \end{tabular}
    \vspace{1em}
    \centering
    \small
    \caption{Costs of changing a feature in \ieeecis dataset}
    \label{tab:ieeecis_costs}
    \begin{tabular}{@{}lc}
    \textbf{Feature} & \textbf{Estimated cost, \$} \\
    \midrule
    \textit{DeviceType} & 0.1 \\
    \textit{P\_emaildomain} & 0.2 \\
    \textit{card\_type} & 20 \\
    \end{tabular}
    \centering
    \small
    \vspace{1em}
    \caption{Costs of changing a feature in \homecredit}
    \begin{tabular}{@{}lc}
    \textbf{Feature} & \textbf{Estimated cost, \$} \\
    \midrule 
    \textit{NAME\_CONTRACT\_TYPE} & 0.1 \\
    \textit{NAME\_TYPE\_SUITE} & 0.1 \\
    \textit{FLAG\_EMAIL} & 0.1 \\
    \textit{WEEKDAY\_APPR\_PROCESS\_START} & 0.1 \\
    \textit{HOUR\_APPR\_PROCESS\_START} & 0.1 \\
    \hline
    \textit{AMT\_INCOME\_TOTAL} & 1 \\
    \hline
    \textit{FLAG\_EMP\_PHONE} & 10 \\
    \textit{FLAG\_WORK\_PHONE} & 10 \\
    \textit{FLAG\_CONT\_MOBILE} & 10 \\
    \textit{FLAG\_MOBIL} & 10 \\
    \hline
    \textit{FLAG\_OWN\_CAR} & 100 \\
    \textit{FLAG\_OWN\_REALTY} & 100 \\
    \textit{REG\_REGION\_NOT\_LIVE\_REGION} & 100 \\
    \textit{REG\_REGION\_NOT\_WORK\_REGION} & 100 \\
    \textit{LIVE\_REGION\_NOT\_WORK\_REGION} & 100 \\
    \textit{REG\_CITY\_NOT\_LIVE\_CITY} & 100 \\
    \textit{REG\_CITY\_NOT\_WORK\_CITY} & 100 \\
    \textit{LIVE\_CITY\_NOT\_WORK\_CITY} & 100 \\
    \textit{NAME\_INCOME\_TYPE} & 100 \\
    \textit{CLUSTER\_DAYS\_EMPLOYED} & 100 \\
    \textit{NAME\_HOUSING\_TYPE} & 100 \\
    \textit{OCCUPATION\_TYPE} & 100 \\
    \textit{ORGANIZATION\_TYPE} & 100 \\
    \hline
    \textit{NAME\_EDUCATION\_TYPE} & 1000 \\
    \textit{NAME\_FAMILY\_STATUS} & 1000 \\
    \textit{HAS\_CHILDREN} & 1000 \\
    \end{tabular}
\label{tab:homecredit_costs}
\end{table}

\paragraph{Credit-Score Manipulation.}\label{sec:credit-score}
In our feature set we include the features that contain credit scores from unspecified external credit-scoring agencies. One reported way of affecting such credit scores is using credit piggybacking~\cite{piggybacking}. During piggybacking, a rating buyer finds a ``donor'' willing to share a credit for a certain fee. We introduce a model that captures costs of manipulating a credit score through piggybacking.

We assume that after one piggybacking manipulation the rating is averaged between ``donor'' and recipient, and that ``donors'' have the maximum rating ($1.0$). Then, the cost associated to increasing the rating from, e.g., $0.5$ to $0.75$ is the same as that of increasing from $0.9$ to $0.95$. This cost cannot be represented by a linear function.
Let the initial score value be $x$. The updated credit score after piggybacking is $x' = \nicefrac{(x + 1)}{2}$.
If we repeat the operation $n$ times, the score becomes:
$$
x' = \frac{x + 2^n - 1}{2^n}
$$
Thus, the number of required piggybacking operations can be computed from the desired final score $x'$ as $
n = \log_2\frac{1-x}{1-x'}$,
and the total cost is $c(x, x') = nC$, where $C$ is the cost of one operation. For the sake of an example, we estimate it to be \$10,000.
$$
c(x, x') = C\log_2\frac{1-x}{1-x'} = C(\log_2(1-x) - \log_2(1-x'))
$$
To represent this cost function for adversarial training, we can use the encoding described in \cref{sec:defenses-relax}, setting $\enc(x) = \log_2(1-x)$. Then, the relaxed cost function becomes $\rcost(x, x') = C|\enc(x) - \enc(x')|$, which is suitable for our defense algorithm. 
This is not a fully realistic model, as we cannot know how exactly credit score agencies compute the rating. However, it is based on a real phenomenon and enables us to demonstrate our framework's support of non-linear costs.

\begin{table}[t]
    \centering
    \caption{Effect of beam size $B$ in the Universal Greedy algorithm on the \ieeecis dataset. The success rates are close for all choices of the beam size, thus the beam size of one offers the best performance in terms of runtime.}
    \label{tab:beam-size}
    \begin{tabular}{r|rrrr}
    \toprule
    {} & \multicolumn{4}{l}{Adv. success, \%} \\
    Cost bound $\rightarrow$ &         10 & 30 & Gain & $\infty$ \\
    Beam size $\downarrow$ &              &       &      &  \\
    \midrule
    1         &        45.32 & \textbf{56.57} & \textbf{56.22} & \textbf{68.20} \\
    10        &        45.32 & 56.01 & 55.65 & 56.01 \\
    100       &        45.32 & 56.53 & 56.18 & 56.53 \\
    \bottomrule
    \end{tabular}
    \\[1em]
    \begin{tabular}{r|rrrr}
    \toprule
    {} & \multicolumn{4}{l}{Success/time ratio} \\
    Cost bound $\rightarrow$ &         10 & 30 & Gain & $\infty$ \\
    Beam size $\downarrow$ &              &       &      &  \\
    \midrule
    1         &               \textbf{3.78} & \textbf{4.80} & \textbf{2.53} & \textbf{2.06} \\
    10        &               2.14 & 2.25 & 1.31 & 1.15 \\
    100       &               0.66 & 0.65 & 0.65 & 0.66 \\
    \bottomrule
    \end{tabular}

\end{table}

\begin{figure*}
    \captionsetup[subfigure]{justification=centering}
    \centering
    \vspace{1em}
    \begin{subfigure}[t]{.5\linewidth}
    \centering
    \includegraphics[width=0.9\linewidth]{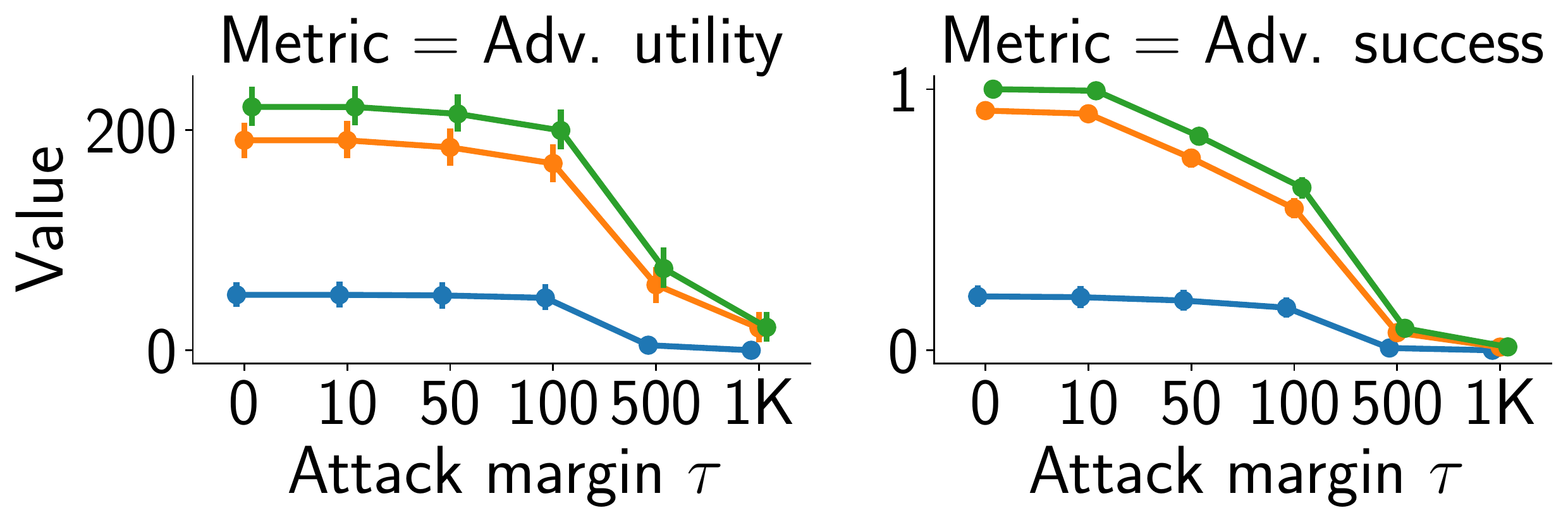} \\
    \vspace{.1em}
    \caption{
    \ieeecis. Model (test acc.): \\
    \textcolor{divergent1}{$\bullet$}~LR~(0.62) \, \textcolor{divergent2}{$\bullet$}~XGBT~(0.83) \, \textcolor{divergent3}{$\bullet$}~TabNet~(0.77)}
    \end{subfigure}~
    \begin{subfigure}[t]{.5\linewidth}
    \includegraphics[width=\linewidth]{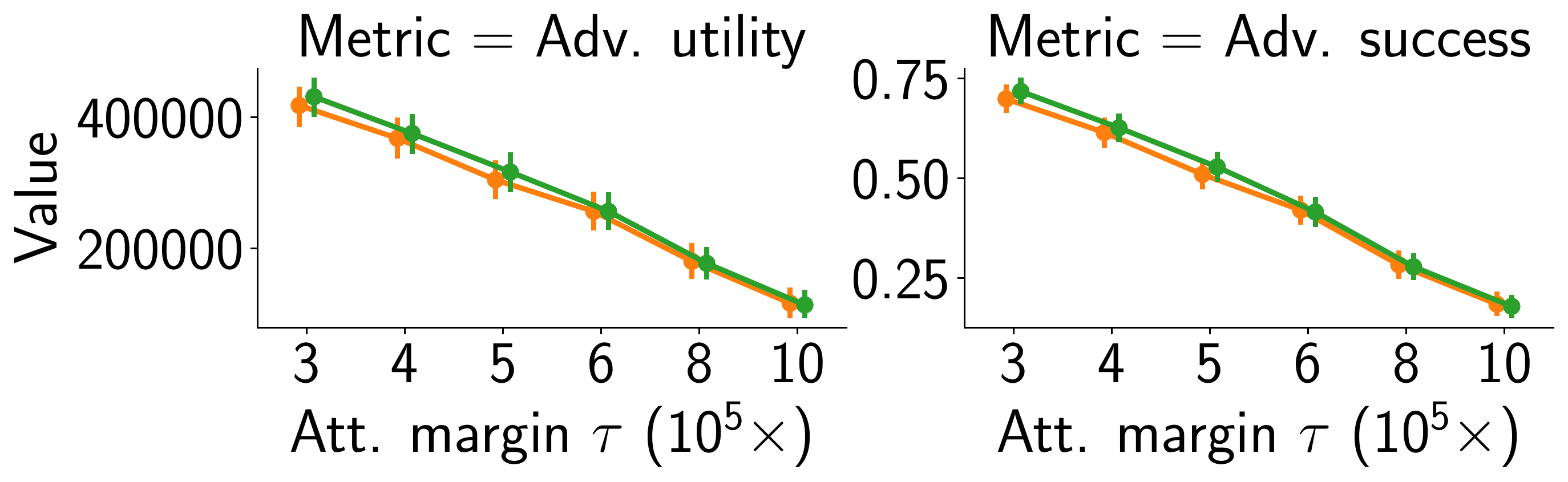}
    \vspace{-1em}
    \caption{\homecredit. Model (test acc.): \\
    \textcolor{divergent2}{$\bullet$}~XGBT~(0.65)\, \textcolor{divergent3}{$\bullet$}~TabNet~(0.68)}
    \end{subfigure}
    \caption{Results of utility-bounded graph-based attacks against three types of models. Left pane: Adversarial utility (higher is better for the adversary). Right pane: See \cref{fig:pgd-comparison}. 
    On \ieeecis, the attack can achieve utility from approximately up to approximately \$200 per attack against TabNet and XGBT. On \homecredit, the average utility ranges between $\$400,000$ and $\$200,000$.}
    \label{fig:attacks-regular-ub}
    \vspace{-1.2em}
\end{figure*}

\section{Variable-Gain Attacks}
\label{sec:var-gain}
In \cref{sec:experiments}, we evaluated the attacks and defenses under the assumption of constant gain (\cref{sec:assumptions}). To check the sensitivity of defense method to this assumption, we conduct a set of experiments to evaluate the defense performance against an adversary which has a capability to modify the gains.
\FloatBarrier
We use the same setup as the \ieeecis evaluation (see \cref{sec:experiments-extra}). As before, we use the value of the \emph{transaction amount} feature to be the gain of an adversarial example. Previously, the adversary was not capable of changing the value of this feature. In these experiments, we allow such changes. The adversarial cost of modifying the feature is equal to the change if the adversary increases the transaction amount, i.e., increasing the transaction amount by \$1 costs \$1. In this setup, cost-bounded adversarial training shows the same trends and demonstrates behaviour similar to the constant gain scenario, while the behaviour of utility-bounded training is different. 

\paragraph{Results.} We follow the protocol in \cref{sec:experiments} to evaluate both regular and adversarially trained models against the utility-bounded attack.
First, we show in \cref{fig:attacks-var} our results for the attack against undefended models. The results are not substantially different from the same experiment with constant gain (\ref{fig:attacks-regular}).
We show in \cref{fig:defense-var} the performance of our defenses against the utility-bounded attack. Cost-bounded adversarial training shows the same trends and demonstrates the behaviour similar to the constant gain scenario (see \cref{fig:defense-cb}). The situation with utility-bounded training is different. Although the robustness is always better after adversarial training, there is no clear trend between defense strength and attack success. Indeed, some models trained to be robust against large $\tau$ (less robust) are in fact more robust than models with low $\tau$. We posit that the reason is that the adversarial training was performed with the assumption of constant gain.

\begin{figure}[t]
    \centering
    \includegraphics[width=.5\linewidth]{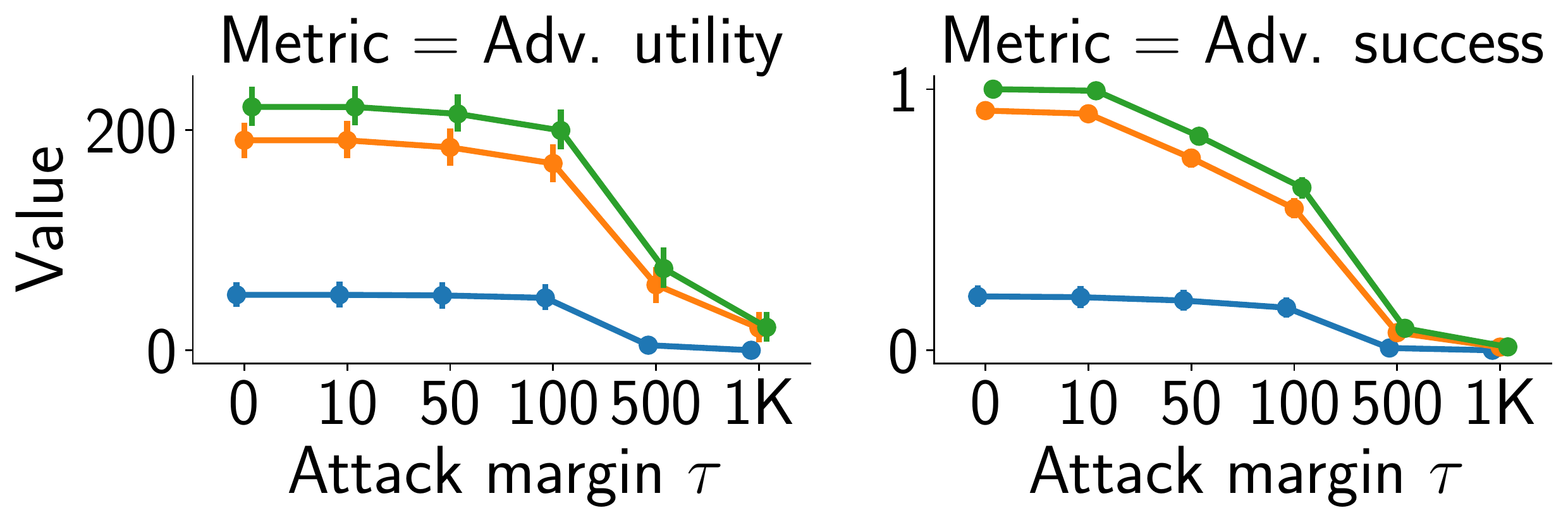} \\
    \vspace{1em}
    \caption{\ieeecis. Model (test acc.): \,
    \textcolor{divergent1}{$\bullet$}~LR (0.62) \, \textcolor{divergent2}{$\bullet$}~XGBT (0.83) \, \textcolor{divergent3}{$\bullet$}~TabNet (0.77). \\
    Utility-bounded attack with variable gain against undefended models. We show the adversary's success and utility (y-axis) versus the adversary's desired utility marge $\tau$ (x-axis). The capability to modify the gain does not change the the utility or success, compared to constant gain (see \cref{fig:attacks-regular-ub}).
    }
    \label{fig:attacks-var}
\end{figure}

\begin{figure*}[t]
    \centering
    \includegraphics[width=0.7\linewidth]{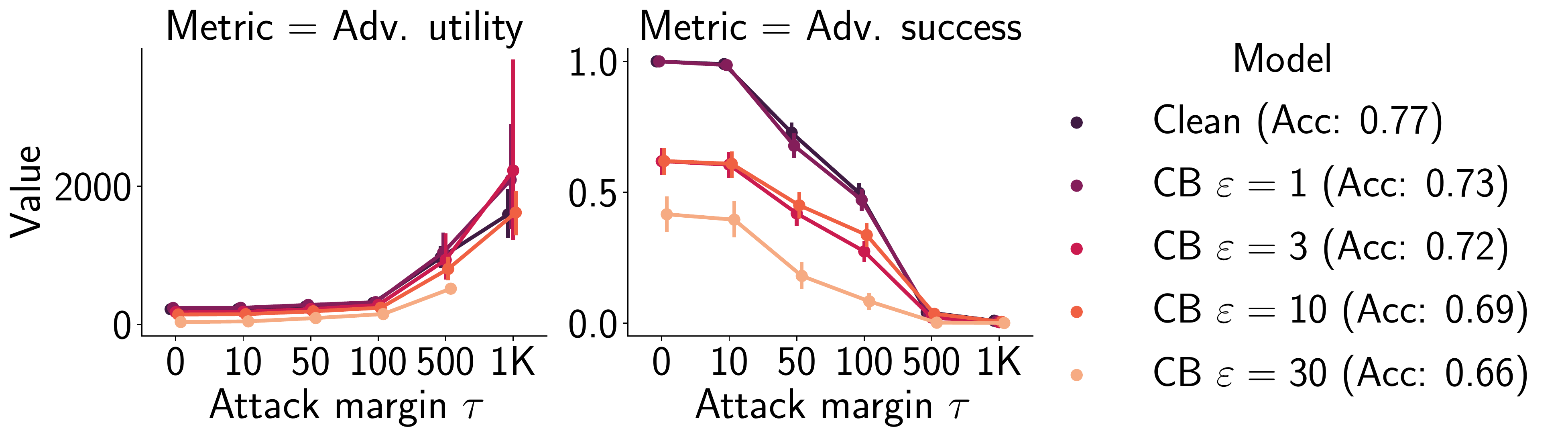}
    \\
    \includegraphics[width=0.7\linewidth]{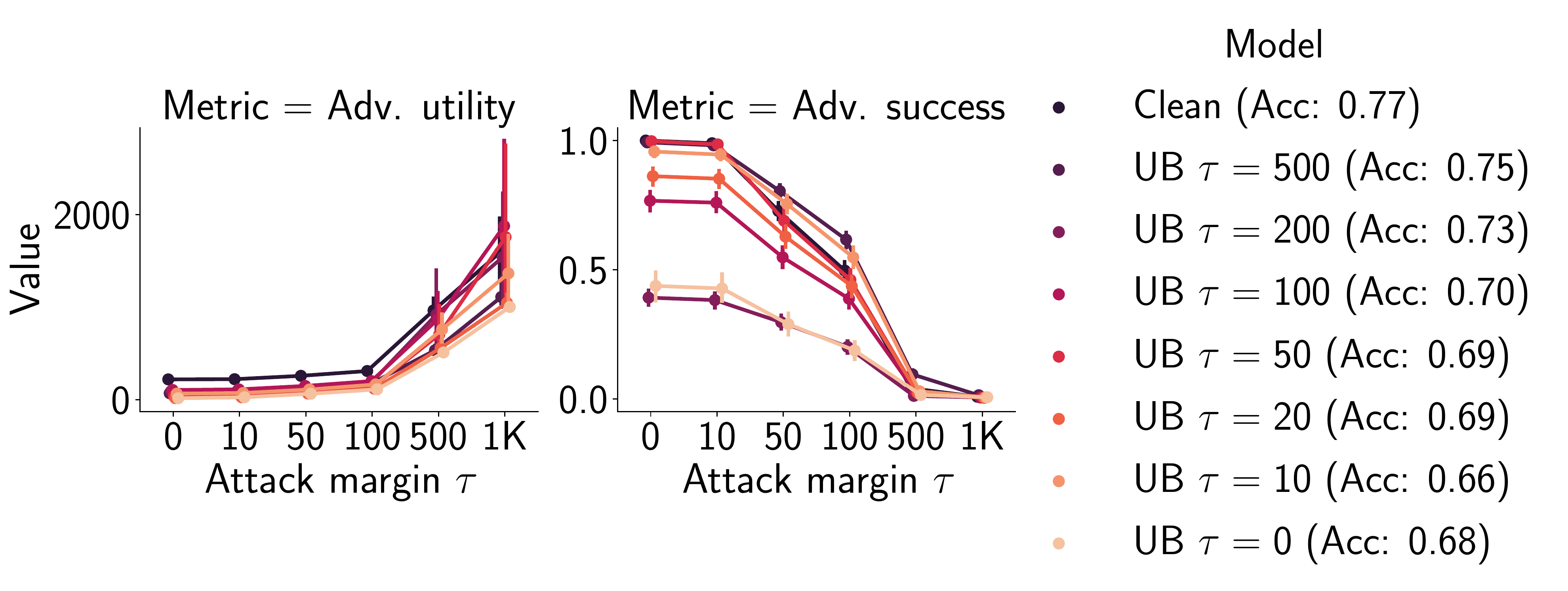}
    \caption{Utility-bounded attack with variable gain against cost-bounded (top) and utility-bounded (bottom) adversarially trained models. We show the adversary’s success and utility (y-axis) versus the adversary’s desired utility margin $\tau$ (x-axis). Both cost-bounded and utility-bounded adversarial training improves robustness compared with the undefended model against adversaries that are capable of modifying the gain.}
    \label{fig:defense-var}
\end{figure*}

\begin{figure*}[t]
    \begin{subfigure}[t]{\linewidth}
    \centering
    \includegraphics[width=0.75\linewidth]{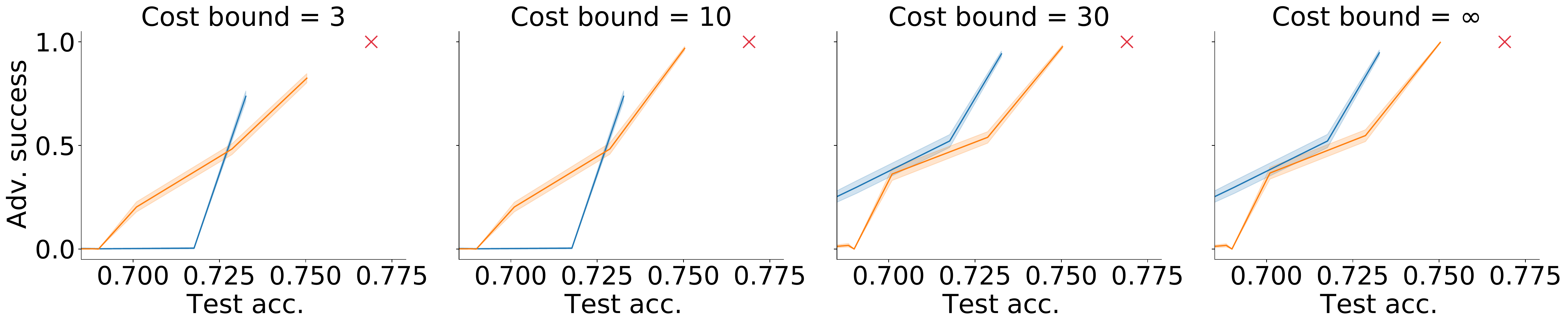}\\[1em]
    \includegraphics[width=0.75\linewidth]{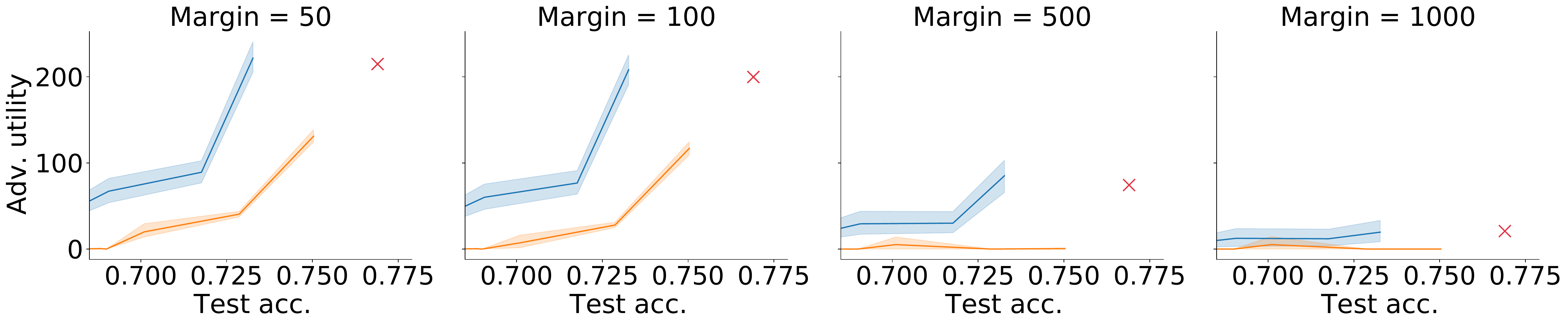}
    \caption{\ieeecis:\quad \textcolor{divergent1}{$\bullet$}~CB-trained models\quad \textcolor{divergent2}{$\bullet$}~UB-trained models\quad \textcolor{rocket3}{$\times$}~Clean model}
    \vspace{1.5em}
    \end{subfigure}
    \begin{subfigure}[t]{\linewidth}
    \centering
    \includegraphics[width=0.75\linewidth]{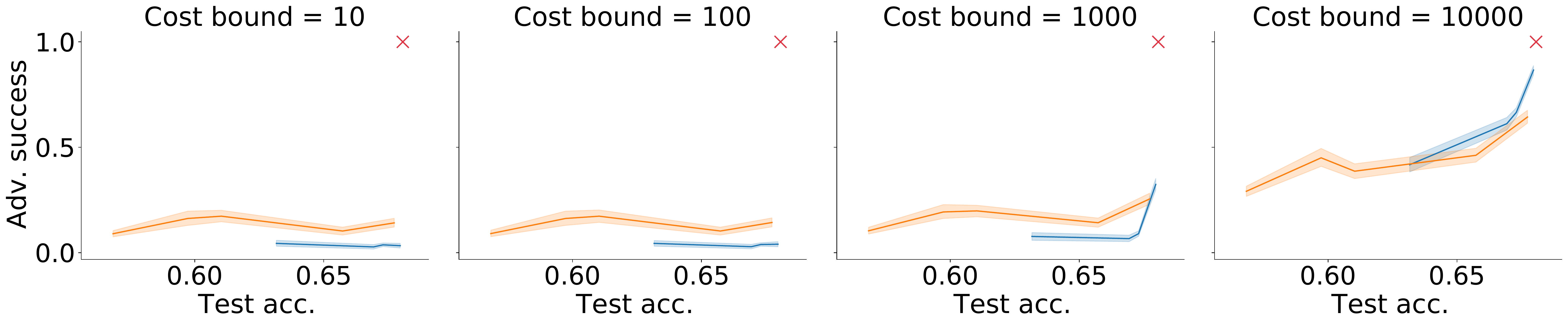}\\[1em]
    \hspace{-.5em}\includegraphics[width=0.75\linewidth]{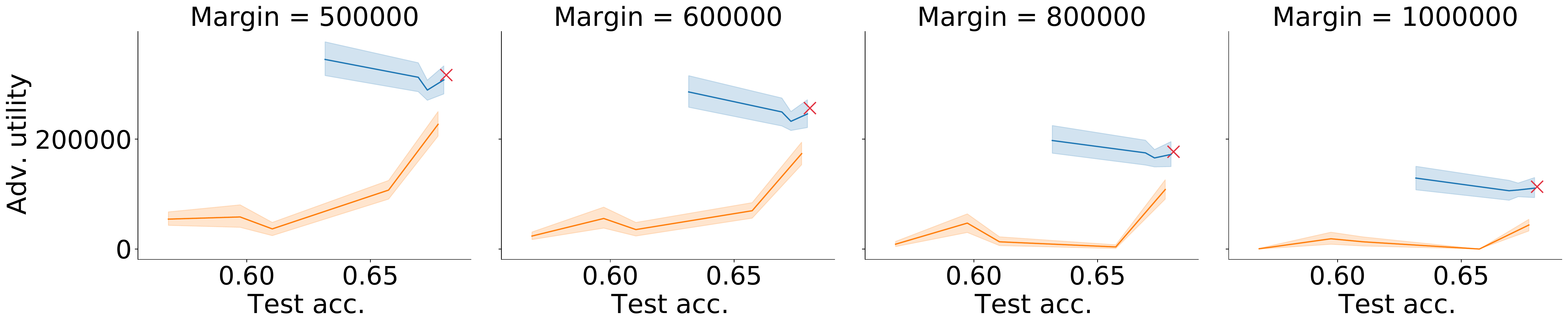}
    \caption{\homecredit:\quad \textcolor{divergent1}{$\bullet$}~CB-trained models\quad \textcolor{divergent2}{$\bullet$}~UB-trained models\quad \textcolor{rocket3}{$\times$}~Clean model}
    \end{subfigure}
    \caption{\textit{Accuracy-robustness and utility-robustness trade-offs} for Cost-bounded and Utility-bounded adversarially trained models. The curves show accuracy (x-axis) and utility and success rate (x-axis) for the utility- and cost-bounded models presented in \cref{fig:defense-cb} and \cref{fig:defense-ub}. When one curve is strictly below the other curve, it provides a better trade-off since it has better robustness for the same accuracy. Utility-bounded models consistently show better trade-offs for all utility-aware attacks. For CB attacks the situation is less consistent: for small cost-bounds CB defense outperforms utility-bounded one, while for the largest budgets utility bounded shows better results.
    }
    \label{fig:defense-tradeoffs}
\end{figure*}

\end{document}